\documentclass[11pt,a4paper]{article}

\usepackage[utf8]{inputenc}
\usepackage{hyperref}
\usepackage{amsthm,amsmath,amsfonts,amssymb}
\usepackage{graphicx}
\usepackage[caption=false,font=footnotesize]{subfig}
\usepackage{booktabs,diagbox}
\usepackage{multirow}
\usepackage{enumerate}
\usepackage{paralist}
\usepackage{acro}
\usepackage[squaren,thickqspace]{SIunits}
\usepackage{todonotes}
\usepackage{color}
\usepackage{cleveref}
\usepackage{algorithm,algpseudocode}

\DeclareAcronym{LDDMM}{
  short = LDDMM,
  long = large deformation diffeomorphic metric mapping}
\DeclareAcronym{RKHS}{
  short = RKHS,
  long = reproducing kernel {H}ilbert space}
\DeclareAcronym{ML}{
  short = ML,
  long = maximum likelihood}
\DeclareAcronym{ET}{
  short = ET,
  long = electron tomography}
\DeclareAcronym{TV}{
  short = TV,
  long = total variation}
\DeclareAcronym{FBP}{
  short = FBP,
  long = filtered back projection}
\DeclareAcronym{ICES}{
  short = ICES,
  long = Institute for Computational Engineering and Sciences}
\DeclareAcronym{CAS}{
  short = CAS,
  long = Chinese Academy of Sciences}
\DeclareAcronym{KTH}{
  short = KTH,
  long =   KTH -- Royal Institute of Technology}
\DeclareAcronym{LSEC}{
  short = LSEC,
  long = State Key Laboratory of Scientific and Engineering Computing}
\DeclareAcronym{ODL}{
  short = ODL,
  long = Operator Discretization Library}
\DeclareAcronym{PDE}{
  short = PDE,
  long = partial differential equation}     
\DeclareAcronym{ODE}{
  short = ODE,
  long = ordinary differential equation}    
\DeclareAcronym{FFT}{
  short = FFT,
  long = Fast Fourier transform} 
\DeclareAcronym{SSIM}{
  short = SSIM,
  long = structural similarity}
\DeclareAcronym{PSNR}{
  short = PSNR,
  long = peak signal-to-noise ratio}
\DeclareAcronym{SNR}{
  short = SNR,
  long = signal-to-noise ratio}  
\DeclareAcronym{PET}{
  short = PET,
  long = positron emission tomography}
\DeclareAcronym{SPECT}{
  short = SPECT,
  long = single photon emission computed tomography}
\DeclareAcronym{CT}{
  short = CT,
  long = computed tomography}
\DeclareAcronym{MRI}{
  short = MRI,
  long = magnetic resonance imaging}
\DeclareAcronym{FDG}{
  short = {${}_{18}$F-FDG},
  long = {$[_{18}$F$]$-fluorodeoxyglucose}}
\DeclareAcronym{MLEM}{
  short = ML-EM,
  long = maximum-likelihood expectation maximisation}
\DeclareAcronym{MAP}{
  short = MAP,
  long = maximum a posteriori}  

\newtheorem{definition}{Definition}

\newtheorem{proposition}{Proposition}

\newcommand{\Natural}{\mathbb{N}}
\newcommand{\Real}{\mathbb{R}}
\newcommand{\domain}{\Omega}
\newcommand{\datadomain}{M}
\newcommand{\VelocitySpace}[2]{L^{#1}\left( [0,1], #2 \right)}
\newcommand{\RecSpace}{X}
\newcommand{\DataSpace}{Y}
\newcommand{\image}{f}
\newcommand{\imagetemp}{I}
\newcommand{\imageother}{J}
\newcommand{\data}{g}
\newcommand{\dataother}{h}
\newcommand{\datanoise}{e}
\newcommand{\target}{I^*}
\newcommand{\template}{I_0}
\newcommand{\vecfield}{\nu}
\newcommand{\vecfieldother}{\mu}
\newcommand{\vfield}{\boldsymbol{\vecfield}}
\newcommand{\vfieldother}{\boldsymbol{\vecfieldother}}
\newcommand{\vfieldtrue}{\vfield^{*}}
\newcommand{\vfieldopt}{\widehat{\vfield}}
\newcommand{\diffeo}{\phi}

\newcommand{\diffeoflow}[2]{\varphi^{#1}_{#2}}
\newcommand{\intenfield}{\zeta}
\newcommand{\intenfieldtrue}{\intenfield^{*}}
\newcommand{\intenfieldopt}{\widehat{\intenfield}}
\newcommand{\templateflow}[2]{I^{#1}_{#2}}
\newcommand{\metatraj}[2]{f^{#1}_{#2}}
\newcommand{\vzero}{\boldsymbol{0}}
\DeclareMathOperator{\Det}{Det}
\DeclareMathOperator{\Diff}{Diff}
\DeclareMathOperator{\SNR}{PSNR}
\newcommand{\der}{\mathrm{d}}

\DeclareMathOperator*{\argmin}{arg\,min}
\DeclareMathOperator{\ForwardOp}{\mathcal{A}}
\DeclareMathOperator{\ObjFunc}{\mathcal{J}}
\DeclareMathOperator{\DataDiscr}{\mathcal{L}}

\DeclareMathOperator{\DeforOp}{\mathcal{W}}
\DeclareMathOperator{\metric}{d_G}
\DeclareMathOperator{\Id}{Id}
\newcommand{\Cdot}{\,\cdot\,}
\newcommand{\captionsubref}[1]{\protect\subref{#1}}

\DeclareMathOperator*{\Div}{div}

\newcommand{\D}{\,\text{D}}

\graphicspath{
  {figurestot/},
  {figurestot/temporal__t_num_angles_10_min_angle_0_max_angle_pi_noise_0_25__sigma_2__lamb_1e_5__tau_1e_5__niter_100__ntimepoints_10datatimepoints10/}
}
\crefname{equation}{}{}
\Crefname{equation}{}{}
\crefname{item}{}{}
\Crefname{item}{}{}
\hypersetup{%
    bookmarks=true,         % show bookmarks bar?
    pdfmenubar=true,       % show Acrobat's menu?
    pdffitwindow=false,     % window fit to page when opened
    pdfstartview={FitH},    % fits the width of the page to the window
    pdftitle={Image reconstruction through metamorphosis},            
    colorlinks=true,       % false: boxed links; true: colored links
    linkcolor=red,          % color of internal links
    citecolor=red,        % color of links to bibliography
    filecolor=blue,      % color of file links
    urlcolor=blue,           % color of external links
    pdfborder = {0,0,0}
}
\addunit{\decibel}{dB}

\title{Image reconstruction through metamorphosis}
\author{Barbara Gris \thanks{LJLL - Laboratoire Jacques-Louis Lions, UPMC, Paris, France.} \and Chong Chen \thanks{LSEC, ICMSEC, Academy of Mathematics and Systems Science, Chinese Academy of Sciences, Beijing 100190, People's Republic of China}  \and Ozan \"Oktem \thanks{Department of Mathematics, KTH -- Royal Institute of Technology, Stockholm, Sweden.}}
\date{}
\begin{document}
\acuse{FDG}  

\maketitle
\begin{abstract}
This article adapts the framework of metamorphosis to solve inverse problems in imaging that includes joint reconstruction and image registration. 
The deformations in question have two components, one that is a geometric deformation moving intensities and the other a deformation of 
intensity values itself, which, e.g., allows for appearance of a new structure. 
The idea developed here is to reconstruct an image from noisy and indirect observations by registering, via metamorphosis, a template to the observed data. Unlike a registration with only geometrical changes, this framework gives good results when intensities of the template are poorly chosen.  We show that this method is a well-defined regularisation method (proving existence, stability and convergence) and present several numerical examples.
\end{abstract}

\section{Introduction}
In \emph{shape based reconstruction} or \emph{spatiotemporal image reconstruction}, a key difficulty is to match an image against an indirectly observed target (\emph{indirect image registration}). This paper provides theory and algorithms for indirect image registration applicable to general inverse problems. Before proceeding, we give a brief overview of these notions along with a short survey of existing results.

\paragraph{Shape based reconstruction}
The goal is to recover shapes of interior sub-structures of an object whereas variations within these is of less importance. 
Examples of such imaging studies are nano-characterisation of specimens by means of electron microscopy or x-ray phase contrast imaging, e.g., nano-characterisation of materials by electron \ac{ET} primarily focuses on the morphology of sub-structures \cite{BlPeBaBa15}. Another example is quantification of sub-resolution porosity in materials by means of  x-ray phase contrast imaging. 

In these imaging applications it makes sense to account for qualitative prior shape information during the reconstruction. Enforcing an exact spatial match between a template and the reconstruction is often too strong since realistic shape information is almost always approximate, so the natural approach is to perform reconstruction assuming the structures are `shape wise similar' to a template. 
 
\paragraph{Spatiotemporal imaging}
Imaging an object that undergoes temporal variation leads to a spatiotemporal reconstruction problem where both the object and its time variation needs to be recovered from noisy time series of measured data. 
An important case is when the only time dependency is that of the object.

Spatiotemporal imaging occurs in medical imaging, see, e.g., \cite{SoDaPa13} for a survey of organ motion models. 
It is particular relevant for techniques like \ac{PET} and \ac{SPECT}, which are used for visualising the distribution of injected radiopharmaceuticals (activity map). 
The latter is an inherently dynamic quantity, e.g., anatomical structures undergo motion, like the motion of the heart and respiratory motion of the lungs and thoracic wall, during the data acquisition. 
Not accounting for organ motion is known to degrade the spatial localisation of the radiotracer, leading to spatially blurred images. 
Furthermore, even when organ motion can be neglected, there are other dynamic processes, such as the uptake and wash-out of radiotracers from body organs.
Visualising such kinetics of the radiotracers can actually be a goal in itself, as in pre-clinical imaging studies related to drug discovery/development.
The term `dynamic' in \ac{PET} and \ac{SPECT} imaging often refers to such temporal variation due to radiotracers kinetics rather than organ movement \cite{GuReSiMaBu10}.

To exemplify the above mentioned issues, consider \ac{SPECT} based cardiac perfusion studies and \ac{FDG}-\ac{PET} imaging of lung nodules/tumours. 
The former needs to account for the beating heart and the latter for respiratory motion of the lungs and thoracic wall.
Studies show a maximal displacement of \unit{23}{\milli\meter} (average 15--\unit{20}{\milli\meter}) due to respiratory motion \cite{ScLe00} and \unit{42}{\milli\meter} (average 8--\unit{23}{\milli\meter}) due to cardiac motion in thoracic \ac{PET} \cite{WaViBe99}. 

%See \cref{sec:NuclearImagingSurvey} for a more detailed review on various approaches for spatiotemporal image reconstruction in nuclear medicine. 

\paragraph{Indirect image registration (matching)}
In image registration the aim is to deform a template image so that it matches a target image, which becomes challenging when the template is allowed to undergo non-rigid deformations.

A well developed framework is diffeomorphic image registration where the image registration is recast as the problem of finding a suitable diffeomorphism that deforms the template into the target image \cite{Yo10,BaBrMi14}.
The underlying assumption is that the target image is contained in the orbit of the template under the group action of diffeomorphisms. 
This can be stated in a very general setting where diffeomorphisms act on various structures, like landmark points, curves, surfaces, scalar images, or even vector/tensor valued images.

The registration problem becomes more challenging when the target is only known indirectly through measured data. 
This is referred to as \emph{indirect image registration}, see \cite{OkChDoRaBa16} for using small diffeomorphic deformations and \cite{Chen:2018aa, hinkle20124d} for adapting the \ac{LDDMM} framework to indirect image registration.

\section{Overview of paper and specific contributions}
The paper adapts the metamorphosis framework \cite{trouve2005metamorphoses} to the indirect image registration setting. Metamorphosis is an extension of the \ac{LDDMM} framework (diffeomorphometry) \cite{Yo10,MiYoTr14} where not only the geometry of the template, but also the grey-scale values undergo diffeomorphic changes. 

We start by recalling necessary theory from \ac{LDDMM}-based indirect registration (\cref{sec:IndirectReg}). 
Using the notions from \cref{sec:IndirectReg}, we adapt the metamorphosis framework to the indirect setting (\cref{sec:MetaInDirReg}).
We show how this framework allows to define a regularization method for inverse problems, satisfying properties of existence, stability and convergence (\cref{sec:RegPropMeta}). The numerical implementation is outlined in \cref{sec:Numericalmpl}.
 We present several numerical examples from 2D tomography, and in particular give a preliminary result for motion reconstruction when the acquisition is done at several time points. We also study the robustness of our methods with respects to the parameters (\cref{sec:TomoExample}).

\section{Indirect diffeomorphic registration}\label{sec:IndirectReg}
\subsection{Large diffeomorphic deformations}\label{Sec:LargeDeformations}
We recall here the notion of large diffeomorphic deformations defined by flows of time-varying vector fields, as formalized in \cite{arguillere2014shape}.

Let $\domain \subset \Real^d$ be a fixed bounded domain and let $\RecSpace := L^2(\domain, \Real)$ represent grey scale images on $\domain$. Next, let $V$ denote a fixed Hilbert space of vector fields on $\Real^d$. 
We will assume $V \subset C^{p}_0 (\domain)$, i.e., the vector fields are supported on $\domain$ and $p$ times continuously differentiable.
Finally, $\VelocitySpace{1}{V}$ denotes the space of time-dependent $V$-vector fields that are integrable, i.e., 
\[ \vfield(t, \Cdot) \in V 
    \quad\text{and}\quad 
   t \mapsto \bigl\Vert \vfield(t, \Cdot) \bigr\Vert_{C^{p}}
   \text{ is integrable on $[0,1]$.}
\]    
Furthermore, we will frequently make use of the following (semi) norm on 
\[ \Vert \vfield \Vert_p := \Bigl( \int_{0}^1 \bigl\Vert \vfield(t,\Cdot) \bigr\Vert_{V}^p \,\der t \Bigr)^{1/p} \]
where $\Vert \Cdot \Vert_V$ is the naturally defined norm based upon the inner product of the Hilbert space $V$ of vector fields.

The following proposition allows one to consider flows of elements in $\VelocitySpace{1}{V}$ and ensures 
that these flows belong to $\Diff^{p}_0(\domain)$ (set of $p$-diffeomorphisms that are supported in $\domain \subset \Real^d$,
and if $\domain$ is unbounded, tend to zero towards infinity).
\begin{proposition}\label{flow1}
Let $\vfield \in \VelocitySpace{1}{V}$ and consider the ordinary differential equation (flow equation):
\begin{equation}\label{Eq:FlowEq}
\begin{cases}
  \dfrac{\der}{\der t} \diffeo(t,x) = \vfield\bigl(t,\diffeo(t,x)\bigr)  & \\[0.75em]
  \diffeo(0,x) = x &
\end{cases}  
\quad\text{for any $x \in \domain$ and $t \in [0,1]$.}
\end{equation}
Then, \cref{Eq:FlowEq} has a unique absolutely continuous solution $\diffeo(t, \Cdot) \in \Diff^{p}_0(\Real^d)$.
\end{proposition}
The above result is proved in \cite{arguillere2014shape} and the unique solution of \cref{Eq:FlowEq} is henceforth called the \emph{flow of $\vfield$}.
We also introduce to notation $\diffeoflow{\vfield}{s,t} \colon \Real^d \to \Real^d$ that refers to 
\begin{equation}\label{eq.FlowDiffeo}
  \diffeoflow{\vfield}{s,t} := \diffeo(t,\Cdot) \circ \diffeo(s,\Cdot)^{-1}  
    \quad\text{for $s,t \in [0,1]$}
\end{equation}
where $\diffeo \colon \domain \to \Real^d$ denotes the unique solution to \cref{Eq:FlowEq}.

As stated next, the set of diffeomorphisms that are given as flows forms a group that is a complete metric space \cite{arguillere2014shape}.%\todo{Insert reference.}.
\begin{proposition}
Let $V \subset C^{p}_0 (\domain)$ ($p \geq 1$) be an admissible \ac{RKHS} and define  
\[ G_V := \Bigl\{ \diffeo \colon \Real^d \to \Real^d \mid \diffeo = \diffeoflow{\vfield}{0,1} \text{ for some $\vfield \in \VelocitySpace{2}{V}$} \Bigr\}. \]
Then $G_V$ forms a sub-group of $\Diff^{p}_0(\Real^d)$ that is a  complete metric space under the metric
\begin{align*}
 \metric(\phi_1, \phi_2) 
   &:= \inf\Bigl\{ \Vert \vfield \Vert_1 : \vfield \in L^1([0, 1], V) \text{ and } \phi_1 = \phi_2 \circ \diffeoflow{\vfield}{0,1} \Bigr\}
\\   
   &= \inf\Bigl\{ \Vert \vfield \Vert_2 : \vfield \in L^1([0, 1], V) \text{ and } \phi_1 = \phi_2 \circ \diffeoflow{\vfield}{0,1} \Bigr\}.
\end{align*}
\end{proposition}

The elements of $G_V$ are called \emph{large diffeomorphic deformations} and $G_V$ acts on $\RecSpace$ via the \emph{geometric group action} that is defined by the operator 
\begin{equation}\label{eq:GeometricGroupAction}
 \DeforOp \colon G_V \times \RecSpace \to \RecSpace 
   \quad\text{where}\quad
   \DeforOp(\diffeo , \template):=\template \circ \diffeo^{-1}. 
\end{equation}

We conclude by stating regularity properties of flows of velocity fields as well as the group action in \cref{eq:GeometricGroupAction}, these will play an important role in what is to follow. The proof is given in \cite{bruveris2015geometry}.%\todo{Insert reference.}.
\begin{proposition}\label{prop:Regularity}
  Assume $V \subset C^{p}_0 (\domain)$ ($p \geq 1$) is a fixed admissible Hilbert space of vector fields 
  on $\domain$ and $\{ \vfield^n \}_n \subset \VelocitySpace{2}{V}$ a sequence that converges weakly to 
  $\vfield \in \VelocitySpace{2}{V}$. Then, the following holds with $\diffeoflow{n}{t}:= \diffeoflow{\vfield^n}{0,t}$:
  \begin{enumerate}
  \item $(\diffeoflow{n}{t})^{-1}$ converges to $(\diffeoflow{\vfield}{0,t})^{-1}$ uniformly w.r.t. $t \in [0,1]$ and uniformly 
    on compact subsets of $\domain \subset \Real^d$.
  \item $\displaystyle{\lim_{n \to \infty}} \Bigl\Vert \DeforOp(\diffeoflow{n}{t}, \template) 
    - \DeforOp(\diffeoflow{\vfield}{0,t}, \template) \Bigr\Vert_{\RecSpace} = 0$
    for any $\image \in \RecSpace$.
  \end{enumerate}
\end{proposition}

\subsection{Indirect image registration}\label{sec:ImageRegistration}
Image registration (matching) refers to the task of deforming a given template image $\template \in \RecSpace$ so that it matches a given 
target image $\target \in \RecSpace$. 

The above task can also be stated in an \emph{indirect} setting, which refers to the case when the template $\template \in \RecSpace$ 
is to be registered against a target $\target \in \RecSpace$ that is only indirectly known through data $\data \in \DataSpace$ where 
\begin{equation}\label{eq:InvProb} 
  \data = \ForwardOp(\target) + \datanoise. 
\end{equation}  
In the above, $\ForwardOp \colon \RecSpace \to \DataSpace$ (forward operator) is known and assumed to be differentiable and 
$\datanoise \in \DataSpace$ is a single sample of a $\DataSpace$-valued random element that denotes the measurement noise 
in the data.

A further development requires specifying what is meant by deforming a template image, and we will henceforth consider 
diffeomorphic (non-rigid) deformations, i.e., diffeomorphisms that deform images by actin g on them through a group action.

\paragraph{\Ac{LDDMM}-based registration}
An example of using large diffeomorphic (non-rigid) deformations for image registration is to minimize the following functional:
\begin{equation*}
 G_V \ni \diffeo \mapsto \frac{\gamma}{2} \metric(\Id, \diffeo )^2 
   + \bigl\Vert \DeforOp(\diffeo, \template) - \target \bigr\Vert^2_{\RecSpace} 
 \quad\text{given $\gamma > 0$.}
\end{equation*}
If $V$ is admissible, then minimizing the above functional on $G_V$ amounts to minimizing the following functional on $\VelocitySpace{2}{V}$ 
\cite[Theorem~11.2 and Lemma~11.3]{Yo10}:
\begin{equation*}
  \VelocitySpace{2}{V} \ni \vfield \mapsto \frac{\gamma}{2} \Vert \vfield \Vert_2^2 
    + \bigl\Vert \DeforOp(\diffeoflow{\vfield}{0,1}, \template) - \image \bigr\Vert^2_{\RecSpace} 
 \quad\text{given $\gamma > 0$.}
\end{equation*}
Such a reformulation is advantageous since $\VelocitySpace{2}{V}$ is a vector space, whereas $G_V$ is not, so it is easier to minimize a functional 
over $\VelocitySpace{2}{V}$ rather than over $G_V$.

The above can be extended to the indirect setting as shown in \cite{Chen:2018aa}, which we henceforth refer to as \emph{\ac{LDDMM}-based indirect registration}. More precisely, the corresponding indirect registration problem can be adressed by minimising the functional
\[  \VelocitySpace{2}{V} \ni \vfield \mapsto \frac{\gamma}{2} \Vert \vfield \Vert_2^2 
      + \DataDiscr\bigl( (\ForwardOp \circ \DeforOp)(\diffeoflow{\vfield}{0,t}, \template), \data \bigr).
\]
Here, $\DataDiscr \colon \DataSpace \times \DataSpace \to \Real$ is typically given by an appropriate affine transform 
of the data negative log-likelihood \cite{BeLaZa08}, so minimizing $\image \mapsto \DataDiscr\bigl( \ForwardOp(\image), \data \bigr)$ 
corresponds to seeking a maximum likelihood solution of \cref{eq:InvProb}.

An interpretation of the above is that the template image $\template$, which is assumed to be given a priori, acts as a \emph{shape prior} 
when solving the inverse problem in \cref{eq:InvProb} and $\gamma>0$ is a regularization parameter that governs the influence of this 
shape priori against the need to fit measured data.  
This interpretation becomes more clear when one re-formulates \ac{LDDMM}-based indirect registration as
\begin{equation}\label{eq:LDDMM:ODE}
\begin{cases}
\displaystyle{\min_{\vfield \in \VelocitySpace{2}{V}}} 
\biggl[ \frac{\gamma}{2} \Vert \vfield \Vert_{2}^2 
+ \DataDiscr\Bigl( (\ForwardOp \circ \DeforOp)\bigl(\diffeo(1,\Cdot), \template \bigr), \data \Bigr)
\biggr]
& \\[1em]
\displaystyle{\frac{d}{dt}} \diffeo(t,x) = \vfield\bigl( t, \diffeo(t,x) \bigr)   
\quad (t,x)\in \domain \times [0,1], & \\[0.5em] 
\diffeo(0,x)= x \quad x \in \domain. &
\end{cases}  
\end{equation}

\section{Metamorphosis-based indirect registration}\label{sec:MetaInDirReg}

\subsection{Motivation}
As shown in \cite{Chen:2018aa}, access to a template that can act as a shape prior can have profound effect in solving challenging 
inverse problem in imaging. As an example, tomographic imaging problems that are otherwise intractable (highly noisy and sparsely sampled data) 
can be successfully addressed using indirect registration even when using a template is far from the target image used for generating 
the data. 

When template has correct topology and intensity levels, then \ac{LDDMM}-based indirect registration with geometric group action 
is remarkably stable as shown in \cite{Chen:2018aa}. Using a geometric group action, however, makes it impossible to create or remove intensity,
e.g., it is not possible to start out from a template with a single isolated structure and deform it to a image with two isolated structures. 
This severely limits the usefulness of \ac{LDDMM}-based indirect registration, e.g., spatiotemporal images (moves) are likely to involve changes 
in both geometry (objects appear or disappear) and intensity. See \cref{Fig:LDDMM:BadIntensities} for an example of how wrong intensity influences the registration.
\begin{figure}[t]
  \centering
  \subfloat[Template.\label{Fig:LDDMM:BadIntensities:Template}]{\includegraphics[width=0.245\textwidth]{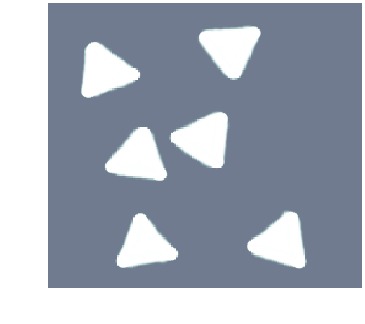}}%  
  \subfloat[Target.\label{Fig:LDDMM:BadIntensities:Target}]{\includegraphics[width=0.245\textwidth]{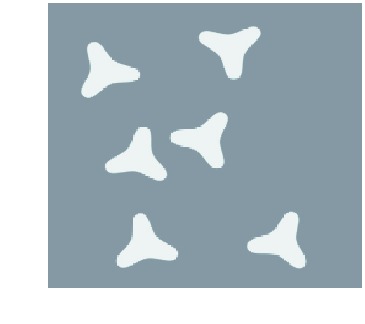}}%
  \subfloat[Reconstruction.\label{Fig:LDDMM:BadIntensities:Reco}]{\includegraphics[width=0.245\textwidth]{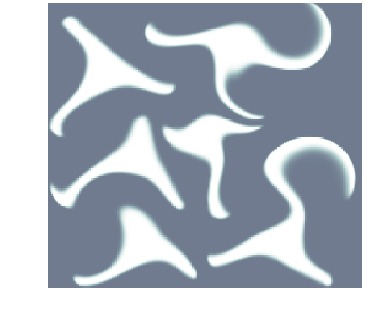}}%
  \hfill
  \subfloat[Data.\label{Fig:LDDMM:BadIntensities:Data}]{\includegraphics[width=0.25\textwidth]{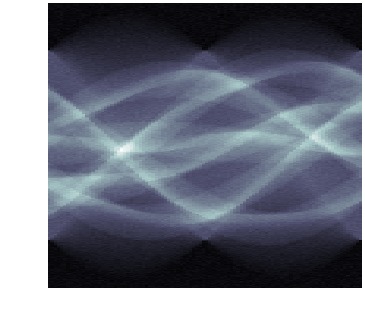}}%  
 \caption{Reconstruction by \ac{LDDMM}-based indirect registration \captionsubref{Fig:LDDMM:BadIntensities:Reco} using a template \captionsubref{Fig:LDDMM:BadIntensities:Template} with a geometry 
   that matches the target \captionsubref{Fig:LDDMM:BadIntensities:Target}, but with incorrect background intensity values. 
   Target is observed indirectly through tomographic data \captionsubref{Fig:LDDMM:BadIntensities:Data}, which is 
   2D parallel beam Radon transform with 100~evenly distributed directions (see \cref{sec:TomoInvProb} for details).
   The artefacts in the reconstruction are due to incorrect background intensity in template.}%
 \label{Fig:LDDMM:BadIntensities}
\end{figure}

As noted in \cite{Chen:2018aa}, one approach is to replace the geometric group action with one that alters intensities, e.g., a mass preserving group action. 
Another is to keep the geometric group action, but replace \ac{LDDMM} with a framework for diffeomorphic deformations that acts on both 
geometry and intensities, e.g., metamorphosis. This latter approach is the essence of metamorphosis-based indirect registration.

\subsection{The metamorphosis framework}
In metamorphosis diffeomorphisms are still generated by flows as in \ac{LDDMM}, but the difference is that they now act with a geometric group 
action on \emph{both} intensities and underlying points. As such, metamorphosis extends \ac{LDDMM}. 
The abstract definition of a metamorphosis reads as follows. 
\begin{definition}[Metamorphosis \cite{trouve2005metamorphoses}]\label{def:Meta}
  Let $V \subset C^{p}_0 (\domain)$ be an admissible Hilbert space and ``$.$'' denotes some group action of $G_V$ on $\RecSpace$.  
  A \emph{Metamorphosis} is a curve $t \mapsto (\diffeo_t, \imageother_t)$ in $G_V \times \RecSpace$. 
  The curve $t \mapsto \image_t := \diffeo_t. \imageother_t$ is called the \emph{image} part, 
  $t \mapsto \diffeo_t$ is the \emph{deformation} part, and 
  $t \mapsto \image_t$ is the \emph{template} part.
\end{definition}
The image part represents the temporal evolution that is not related to intensity changes, i.e., evolution of underlying geometry, 
whereas the template part is the evolution of the intensity. Both evolutions, which are combined in metamorphosis, 
are driven by the same underlying flow of diffeomorphisms in $G_V$.

A important case is when the metamorphosis $t \mapsto (\diffeo_t, \image_t)$ has a deformation part that solves the flow 
equation \cref{Eq:FlowEq} and a template part is $C^1$ in time. More precisely, $\VelocitySpace{2}{\RecSpace}$ denotes the space of 
functions in $\RecSpace$ that are square integrable, i.e.,
\[ \intenfield(t, \Cdot) \in \RecSpace 
    \quad\text{and}\quad 
   t \mapsto \bigl\Vert \intenfield(t, \Cdot) \bigr\Vert_{\RecSpace}
   \text{ is in $L^2([0,1],\Real)$.}
\] 
The norm on $\VelocitySpace{2}{\RecSpace}$ is then 
\[ \Vert \intenfield \Vert_2 := \Bigl( \int_{0}^1 \bigl\Vert \intenfield(t,\Cdot) \bigr\Vert_{\RecSpace}^2 \,\der t \Bigr)^{1/2}. \]
We will also use the notation 
\[  \VelocitySpace{2}{V \times\RecSpace} := \VelocitySpace{2}{V} \times \VelocitySpace{2}{\RecSpace}. \]
Bearing in mind the above notation, for given $(\vfield,\intenfield) \in \VelocitySpace{2}{V \times \RecSpace}$ and $\template \in \RecSpace$, 
define the curve  $t \mapsto \templateflow{\vfield,\intenfield}{t}$, which is absolutely continuous on $[0,1]$, as the solution to 
\begin{equation}\label{Eq:TemplateEvolution}
\begin{cases}
  \dfrac{\der}{\der t} \templateflow{\vfield,\intenfield}{t}(x) = \intenfield\bigl(t, \diffeoflow{\vfield}{0,t}(x)\bigr) & \\[0.75em]
  \templateflow{\vfield,\intenfield}{0}(x) = \template(x)
\end{cases}
\quad\text{with $\diffeoflow{\vfield}{0,t} \in G_V$ as in \cref{eq.FlowDiffeo}.}
\end{equation}
The metamorphosis can now be parametrised as $t \mapsto (\diffeoflow{\vfield}{0,t}, \templateflow{\vfield,\intenfield}{t})$.

\paragraph{Indirect registration}
The indirect registration problem in \cref{sec:ImageRegistration} can be approached by metamorphosis instead of \ac{LDDMM}.  
Similar to \ac{LDDMM}-based indirect image registration in \cite{Chen:2018aa},  
we define \emph{metamorphosis-based indirect image registration} as the minimization of the objective functional
\[  
  \ObjFunc_{\gamma, \tau}(\Cdot; \data) \colon \VelocitySpace{2}{V \times \RecSpace} \to \Real
\]
defined as 
\begin{equation}\label{Eq:ObjectiveFunctional}
 \ObjFunc_{\gamma, \tau}(\vfield, \intenfield; \data) 
   := \frac{\gamma}{2} \Vert \vfield \Vert_2^2 + \frac{\tau}{2} \Vert \intenfield \Vert_2^2 
        + \DataDiscr\Bigl(\ForwardOp\bigl(\DeforOp(\diffeoflow{\vfield}{0,1},\templateflow{\vfield,\intenfield}{1})\bigr),\data \Bigr)
\end{equation}
for given regularization parameters $\gamma, \tau \geq 0$, measured data $\data \in \DataSpace$, and initial template $\template \in \RecSpace$ that sets the initial condition $\templateflow{\vfield,\intenfield}{0}(x) := \template(x)$.

Hence, performing metamorphosis-based indirect image registration of a template $\template$ against a target indirectly observed 
through data $\data$ amounts to 
solving 
\begin{equation}\label{eq.MetaOptim}
  (\vfieldopt, \intenfieldopt) \in \argmin_{(\vfield, \intenfield)} \ObjFunc_{\gamma, \tau}(\vfield, \intenfield; \data).
\end{equation}
The above always has a solution assuming the data discrepancy and the forward operator fulfills some weak requirements (see \cref{prop:Existence}).
From a solution we then obtain the following:
\begin{itemize}
\item \textsl{Initial template:} $\template \in \RecSpace$ such that $\templateflow{\vfield,\intenfield}{0}:= \template$.
\item \textsl{Reconstruction:} 
  Final registered template
  $\image^{\vfieldopt,\intenfieldopt}_1 = \DeforOp\bigl(\diffeoflow{\vfieldopt}{0,1}, \templateflow{\vfieldopt,\intenfieldopt}{1}\bigr) \in \RecSpace$.
\item \textsl{Image trajectory:}
  The evolution of both geometry and intensity of the template, given by 
   $t \mapsto \DeforOp\bigl(\diffeoflow{\vfieldopt}{0,t}, \templateflow{\vfieldopt,\intenfieldopt}{t}\bigr)$.
\item \textsl{Template trajectory:}
  The evolution of intensities of the template, i.e., the part that does not include evolution of geometry: 
  $t \mapsto \templateflow{\vfieldopt, \intenfieldopt}{t}$.
\item \textsl{Deformation trajectory:}
  The geometric evolution of the template, i.e., the part that does not include evolution of intensity:
  $t \mapsto \DeforOp(\diffeoflow{\vfieldopt}{0,t}, \template)$.
\end{itemize}

\subsection{Regularising properties}\label{sec:RegPropMeta}
In the following we prove several properties (existence, stability and convergence) of metamorphosis-based indirect image registration, 
which are necessary if the approach is to constitute a \emph{well defined regularisation method} (notion defined in \cite{grasmair2010generalized}). 
We set $\RecSpace := L^2(\domain, \Real)$ and $Y$ a Hilbert space.

\begin{proposition}[Existence]\label{prop:Existence}
  Assume $\ForwardOp \colon \RecSpace \to \DataSpace$%\todo{$\RecSpace=L^2$, right? What about $\DataSpace$?}
  is continuous and the data discrepancy $\DataDiscr(\Cdot, \data) \colon \DataSpace \to \Real$ is weakly lower semi-continuous 
  for any $\data \in \DataSpace$. Then, $\ObjFunc_{\gamma, \tau}(\Cdot, \data) \colon \VelocitySpace{2}{V \times \RecSpace} \to \Real$ 
  defined through  \cref{Eq:ObjectiveFunctional,Eq:TemplateEvolution} has a minimizer in $\VelocitySpace{2}{V \times \RecSpace}$ 
  for any $\template \in L^2(\domain,\Real)$. 
\end{proposition}
\begin{proof}
We follow here the strategy to prove existence of minimal trajectories for metamorphosis (as in \cite{charon2016metamorphoses} for instance).
One considers a minimizing sequence of $\ObjFunc_{\gamma, \tau}(\Cdot; \data)$, i.e., a sequence that converges to the infimum of 
$\ObjFunc_{\gamma, \tau}(\Cdot; \data)$ (such a sequence always exists). The idea is to prove that such a minimizing sequence has a 
sub-sequence that converges to a point in $\VelocitySpace{2}{V \times \RecSpace}$, i.e., the infimum is contained in 
$\VelocitySpace{2}{V \times \RecSpace}$ which proves existence of a minima.

Bearing in mind the above, we start by considering a minimizing sequence $\bigl\{ (\vfield^n, \intenfield^n) \bigr\}_n \subset \VelocitySpace{2}{V \times \RecSpace}$ to $\ObjFunc_{\gamma, \tau}(\Cdot; \data)$, i.e.,  
\[  \lim_{n \to \infty} \ObjFunc_{\gamma, \tau}(\vfield^n, \intenfield^n; \data) 
      = \inf_{\vfield, \intenfield} \ObjFunc_{\gamma, \tau}(\vfield, \intenfield; \data). 
\]     
Since $\bigl\{ \vfield^n \bigr\}_n \subset \VelocitySpace{2}{V}$ is bounded, it has a sub-sequence that converges to an element 
$\vfield^{\infty} \in \VelocitySpace{2}{V}$. Likewise, $\bigl\{ \intenfield^n \bigr\}_n \subset \VelocitySpace{2}{\RecSpace}$ has a sub-sequence 
that converges to an element $\intenfield^{\infty} \in \VelocitySpace{2}{\RecSpace}$. Hence, with a slight abuse of notation, we 
conclude that  
\[  \vfield^n \rightharpoonup \vfield^{\infty} \quad\text{and}\quad 
     \intenfield^n \rightharpoonup \intenfield^{\infty}
     \quad\text{as $n \to \infty$.}
\]
The aim is now to prove existence of minimizers by showing that $(\vfield^{\infty}, \intenfield^{\infty})$ is a minimizer to 
$\ObjFunc_{\gamma, \tau}(\Cdot; \data) \colon \VelocitySpace{2}{V \times \RecSpace} \to \Real$.

Before proceeding, we introduce some notation in order to simplify the expressions. Define 
\begin{equation}\label{eq:Notation}
  \templateflow{n}{t} := \templateflow{\vfield^n,\intenfield^n}{t}
  \quad\text{and}\quad
  \diffeoflow{n}{s,t} := \diffeoflow{\vfield^n}{s,t}
  \quad\text{for $n \in \Natural \bigcup \{ \infty \}$.}
\end{equation}
Hence, assuming geometric group action \cref{eq:GeometricGroupAction} and using \cref{eq.FlowDiffeo}, we can write  
\[
 \ObjFunc_{\gamma, \tau}(\vfield^n, \intenfield^n; \data) 
   = \frac{\gamma}{2} \Vert \vfield^n \Vert_2^2 + \frac{\tau}{2} \Vert \intenfield^n \Vert_2^2 
        + \DataDiscr\Bigl( \ForwardOp\bigl( \templateflow{n}{1} \circ \diffeoflow{n}{1,0} \bigr), \data\Bigr)
\]
for $n \in \Natural \bigcup \{ \infty \}$. Assume next that the following holds: 
\begin{equation}\label{eq:Claim1}
 \templateflow{n}{1} \circ \diffeoflow{n}{1,0} \rightharpoonup \templateflow{\infty}{1} \circ \diffeoflow{\infty}{1,0}
    \quad\text{as $n \to \infty$.}
\end{equation}   
The data discrepancy term $\DataDiscr(\Cdot, \data) \colon \DataSpace \to \Real$ is weakly lower semi continuous and 
the forward operator $\ForwardOp \colon \RecSpace \to \DataSpace$ is continuous, so $\DataDiscr(\Cdot, \data) \circ \ForwardOp$ is also weakly lower semi continuous and then \cref{eq:Claim1} implies
\begin{equation}\label{eq:Claim2}
 \DataDiscr\bigl(\ForwardOp(\templateflow{\infty}{1} \circ \varphi^{\infty}_{1,0} ), \data \bigr) 
 \leq 
 \liminf_{n \to \infty} \DataDiscr(\ForwardOp(\templateflow{n}{1} \circ \varphi^n_{1,0} ), \data).  
\end{equation}   
Furthermore, from the weak convergences of $\vfield^n$ and $\intenfield^n$, we get
\begin{equation}\label{eq:Claim3}
 \frac{\gamma}{2} \Vert \vfield^{\infty} \Vert_2^2 + \frac{\tau}{2} \Vert \intenfield^{\infty} \Vert_2^2 
    \leq 
    \liminf_{n \to \infty}  \Bigl[ \frac{\gamma}{2} \Vert \vfield^n \Vert_2^2 + \frac{\tau}{2} \Vert \intenfield^n \Vert_2^2 \Bigr].
\end{equation}   
Hence, combining \cref{eq:Claim2,eq:Claim3} we obtain
\[
   \ObjFunc_{\gamma, \tau}(\vfield^{\infty}, \intenfield^{\infty}; \data) 
   \leq \displaystyle{\lim_{n \to \infty}} \ObjFunc_{\gamma, \tau}(\vfield^n, \intenfield^n; \data).
\]
Since $\bigl\{ (\vfield^n, \intenfield^n) \bigr\}_n \subset \VelocitySpace{2}{V \times \RecSpace}$ is a minimizing sequence, this yields 
\[ 
\ObjFunc_{\gamma, \tau}(\vfield^{\infty}, \intenfield^{\infty}; \data) = \inf_{(\vfield, \intenfield) \in \VelocitySpace{2}{V \times \RecSpace}} 
\ObjFunc_{\gamma, \tau}(\vfield, \intenfield; \data),
\]
which proves $(\vfield^{\infty}, \intenfield^{\infty}) \in \VelocitySpace{2}{V \times \RecSpace}$ is a minimizer to $\ObjFunc_{\gamma, \tau}(\Cdot; \data)$.

Hence, to finalize the proof we need to show that \cref{eq:Claim1} holds. We start by observing that the solution of \cref{Eq:TemplateEvolution} 
can be written as  
\begin{equation}\label{Eq:TemplateEvolutionSolution}
  \templateflow{n}{t} := \templateflow{n}{0}(x) + \int_0^t \intenfield^n\bigl(s, \diffeoflow{n}{0,s}(x) \bigr) \der s  
  \quad
  \text{for $n \in \Natural \cup \{\infty \}$,}
\end{equation}
and note that $(t,x) \mapsto \templateflow{n}{t}(x) \in C([0,1] \times \domain, \Real)$. 
Next, we claim that 
\[ \templateflow{n}{1} \rightharpoonup \templateflow{\infty}{1} \quad\text{for some $\templateflow{\infty}{1} \in \RecSpace$,} \]
which is equivalent to
\begin{equation}\label{eq:Claim1:p1}
 \lim_{n \to \infty} \langle \templateflow{n}{1} - \templateflow{\infty}{1}, \imageother \rangle = 0 
   \quad\text{for any $\imageother \in L^2(\domain, \Real)$.}
\end{equation}   
To prove \cref{eq:Claim1:p1}, note first that since continuous functions are dense in $L^2$, it is enough to show \cref{eq:Claim1:p1} holds for 
$\imageother \in C_0(\domain, \Real)$. Next, 
\begin{align}
\langle \templateflow{n}{1} - \templateflow{\infty}{1}, \imageother \rangle 
  &= \int_\domain \int_0^t \Bigl(\intenfield^n\bigl( s, \diffeoflow{n}{0,s}(x) \bigr) - \intenfield^{\infty}\bigl( s, \diffeoflow{\infty}{0,s}(x) \bigr) \Bigr) \imageother(x) \der s \der x \\
  &= \int_\domain \int_0^t \Bigl(\intenfield^n\bigl( s, \diffeoflow{n}{0,s}(x) \bigr) - \intenfield^n\bigl( s, \diffeoflow{\infty}{0,s}(x) \bigr) \Bigr) \imageother(x) \der s \der x \label{Eq:ImageEqTerm1} \\  
  &\quad 
  + \int_\domain \int_0^t \Bigl(\intenfield^n\bigl( s, \diffeoflow{n}{0,s}(x) \bigr) - \intenfield^{\infty}\bigl( s, \diffeoflow{\infty}{0,s}(x) \bigr) \Bigr) \imageother(x) \der s \der x. \label{Eq:ImageEqTerm2}   
\end{align}
Let us now take a closer look at the term in \cref{Eq:ImageEqTerm1}:
\begin{multline*}
\int_\domain \int_0^t \Bigl( \intenfield^n\bigl(s, \diffeoflow{n}{0,s}(x)\bigr)  
    - \intenfield^n\bigl(s, \diffeoflow{\infty}{0,s}(x)\bigr) \Bigr) \imageother(x) \der s \der x 
\\\shoveleft{\qquad 
  = \int_\domain \int_0^t \intenfield^n(s,x) \imageother\bigl(\diffeoflow{n}{0,s}(x)\bigr) 
         \bigl\vert D\diffeoflow{n}{0,s}(x) \bigr\vert \der s \der x 
}\\
    -  \int_\domain \int_0^t \intenfield^{\infty}(s,x) \imageother\bigl(\diffeoflow{\infty}{0,s}(x)\bigr) 
         \bigl\vert D\diffeoflow{\infty}{0,s}(x) \bigr\vert \der s \der x 
\\\shoveleft{\qquad      
 =  \int_\domain \int_0^t \intenfield^n(s,x) \Big( \imageother\bigl(\diffeoflow{n}{0,s}(x) \bigr) 
        \bigl\vert D\diffeoflow{n}{0,s}(x) \bigr\vert - \imageother\bigl(\diffeoflow{\infty}{0,s}(x)\bigr) 
        \bigl\vert D\diffeoflow{\infty}{0,s}(x)\bigr\vert 
     \Big) \der s \der x 
}\\
     - \int_\domain \int_0^t \Big( \intenfield^{\infty}(s,x) - \intenfield^n(s,x) \Big) 
         \imageother\bigl(\diffeoflow{\infty}{0,s}(x)\bigr) 
         \bigl\vert D\diffeoflow{\infty}{0,s}(x) \bigr\vert \der s \der x \\
 =   \langle \intenfield^n, \imageother^n - \imageother^{\infty} \rangle 
    - \langle \intenfield^{\infty} - \intenfield^n , \imageother^{\infty} \rangle
\end{multline*}
where $\imageother^n \in \VelocitySpace{2}{\RecSpace}$ is defined as 
\begin{equation}\label{eq:Jfunc}
  \imageother^n(s,x) :=  \imageother\bigl(\diffeoflow{n}{s,0}(x)\bigr) \bigl\vert D \diffeoflow{n}{s,0}(x) \bigr\vert
    \quad\text{for $n \in \Natural \bigcup \{ \infty \}$.}
\end{equation}
By \cref{prop:Regularity} we know that $\diffeoflow{n}{s,0} \to \diffeoflow{\infty}{s,0}$ and $D \diffeoflow{n}{s,0} \to D \diffeoflow{\infty}{s,0}$ 
uniformly on $\domain$. Since $\imageother$ is continuous on $\domain$, we conclude that $\Vert  \imageother^n - \imageother^{\infty} \Vert_2 $ tends to $0$.
Since $\intenfield^n$ is bounded, we conclude that 
\[  \langle \intenfield^n, \imageother^n - \imageother^{\infty} \rangle 
    \leq 
    \Vert \intenfield^n\Vert_2 \cdot \Vert  \imageother^n - \imageother^{\infty} \Vert_2  
    \to 0.  
\]
Furthermore, since $\intenfield^n \rightharpoonup \intenfield^{\infty}$, we also get 
$\langle \intenfield^{\infty} -   \intenfield^n, \imageother^{\infty} \rangle \to 0$. 
Hence, we have shown that \cref{Eq:ImageEqTerm1} tends to zero, i.e., 
\[ \lim_{n  \to \infty} 
   \int_\domain \int_0^t \Bigl( \intenfield^n\bigl(s,\diffeoflow{n}{0,s}(x)\bigr) - \intenfield^n\bigl(s,\diffeoflow{\infty}{0,s}(x)\bigr) \Bigr) 
   \imageother(x) \der s \der x = 0.
\]   
Finally, we consider the term in \cref{Eq:ImageEqTerm2}. Since $ \intenfield^n \rightharpoonup \intenfield^{\infty}$, we immediately obtain 
\begin{equation*}
 \int_\domain \int_0^t \Bigl( \intenfield^n\bigl(s,\varphi^{\infty}_s(x)\bigr) - \intenfield^{\infty}\bigl(s,\varphi^{\infty}_s(x)\bigr) \Bigr) \imageother(x) \der s \der x  
 = \bigl\langle \intenfield^n - \intenfield^{\infty}, \imageother^{\infty} \bigr\rangle \to 0.
\end{equation*}
To summarise, we have just proved that both terms \cref{Eq:ImageEqTerm1,Eq:ImageEqTerm2} tend to $0$ as $n \to \infty$,
which implies that \cref{eq:Claim1:p1} holds, i.e., $\templateflow{n}{1} \rightharpoonup \templateflow{\infty}{1}$.

To prove \cref{eq:Claim1}, i.e., $\templateflow{n}{1} \circ \varphi^n_{1,0} \rightharpoonup \templateflow{\infty}{1} \circ \varphi^{\infty}_{1,0}$,
we need to show that 
\begin{equation}\label{eq:Claim1simpl}
  \lim_{n \to \infty} \bigl\langle \templateflow{n}{1} \circ \varphi^n_{1,0} 
   -  \templateflow{\infty}{1} \circ \varphi^{\infty}_{1,0}, \imageother \bigr\rangle 
  = 0
  \quad\text{for any $\imageother \in L^2(\domain, \Real)$,}
\end{equation}
and as before, we may assume $\imageother \in C_0 (\domain, \Real)$. Using \cref{eq:Jfunc} we can express the term in 
\cref{eq:Claim1simpl} whose limit we seek as
\begin{multline*}
\bigl\vert  \langle \templateflow{n}{1} \circ \varphi^n_{1,0} -  \templateflow{\infty}{1} \circ \varphi^{\infty}_{1,0}, \imageother \rangle \bigr\vert  
\\ 
\leq  
  \Bigl\vert 
  \bigl\langle 
    \templateflow{n}{1}, \imageother^n(1,\Cdot) - \imageother^{\infty}(1,\Cdot)
  \bigr\rangle 
  \Bigr\vert  
+ \Bigl\vert  \bigl\langle 
     \templateflow{n}{1} - \templateflow{\infty}{1}, \imageother^{\infty}(1,\Cdot)
    \bigr\rangle \Bigr\vert 
\\
\leq 
\Vert \templateflow{n}{1}\Vert  \cdot 
\bigl\Vert \imageother^n(1,\Cdot) - \imageother^{\infty}(1,\Cdot) \bigr\Vert  
+  \bigl\vert \langle \templateflow{n}{1} - \templateflow{\infty}{1},  \imageother^{\infty}(1,\Cdot) \rangle \bigr\vert.
\end{multline*}
Since $\Vert \templateflow{n}{1}\Vert$ is bounded (because $\Vert \intenfield^n\Vert$ is bounded) and since $\templateflow{n}{1} \rightharpoonup \templateflow{\infty}{1}$ (which we shoed before), all terms above tend to $0$ as $n \to \infty$, i.e., \cref{eq:Claim1simpl} holds.

This concludes the proof of \cref{eq:Claim1}, which in turn implies the existence of a minimizer of $\ObjFunc_{\gamma, \tau}(\Cdot; \data)$.
\end{proof}

Our next result shows that the solution to the indirect registration problem is (weakly) continuous w.r.t. variations in data, and as such, it is a 
kind of stability result.
\begin{proposition}[Stability]\label{prop:Stability}
Let $\{ \data_k \}_k \subset \DataSpace$ and assume this sequence converges (in norm) to some $\data \in \DataSpace$. 
Next, for each $\gamma, \tau >0$ and each $k$, define $(\vfield^k, \intenfield^k) \in \VelocitySpace{2}{V \times \RecSpace}$ as 
\[ (\vfield^k, \intenfield^k) = \argmin_{(\vfield, \intenfield)} \ObjFunc_{\gamma, \tau}(\vfield, \intenfield; \data_k). \]
Then there exists a sub sequence of $(\vfield^k, \intenfield^k)$ that converges weakly to a minimizer of $\ObjFunc_{\gamma, \tau}(\Cdot; \data)$
in \cref{Eq:ObjectiveFunctional}.
\end{proposition}
%\todo[inline]{Stability means a regularised solution, in this case a minimiser to $\ObjFunc_{\gamma, \tau}(\Cdot; \data)$, is continuous w.r.t. $\data$ 
%  using a relevant topology ($\gamma, \tau$ are fixed). Current formulation in \cref{prop:Stability} requires a reader to think on how these 
%  claims imply stability. Better to reformulate claims in \cref{prop:Stability} so that they directly refer to stability.} 
\begin{proof}
$\ObjFunc_{\gamma, \tau}(\Cdot; \data_k)$ has a minimizer $(\vfield^k,\intenfield^k) \in \VelocitySpace{2}{V \times \RecSpace}$ for any $\data_k \in \DataSpace$ (\cref{prop:Existence}). 
The idea is first to show that the sequences $(\vfield^k)_k$ and $(\intenfield^k)_k$ are bounded. 
Next, we show that there exists a weakly converging subsequence of $(\vfield^k, \intenfield^k)$ that converges to a minimizer $(\vfield,\intenfield)$ of $\ObjFunc_{\gamma, \tau}(\Cdot; \data)$, which also exists due to \cref{prop:Existence}.

Since $(\vfield^k, \intenfield^k) $ minimizes $\ObjFunc_{\gamma, \tau}(\Cdot; \data_k)$, by \cref{Eq:ObjectiveFunctional} we have 
\begin{equation}\label{eq:BoundedVfield}
  \Vert \vfield^k \Vert_2^2 
    \leq \frac{2}{\gamma} \ObjFunc_{\gamma, \tau}(\Cdot; \data_k)(\vfield^k, \intenfield^k) 
    \leq \frac{2}{\gamma} \ObjFunc_{\gamma, \tau}(\Cdot; \data_k)(\vzero, 0)
    \quad\text{for each $k$.}
\end{equation}    
%\todo{How do you get the second inequality in \cref{eq:BoundedVfield}? It is a minimizer.}
Observe now that if $\vfield=\vzero$ and $\intenfield=0$, then $\diffeoflow{\vfield}{0,1}=\Id$ by \cref{Eq:FlowEq} and $\templateflow{\vfield,\intenfield}{1} = \template$ by \cref{Eq:TemplateEvolution}, so in particular 
\[ \DeforOp\bigl(\diffeoflow{\vfield}{0,1}, \templateflow{\vfield,\intenfield}{1})\bigr) = \template
   \quad\text{whenever $\vfield=\vzero$ and $\intenfield=0$,}
\]
Hence, $\ObjFunc_{\gamma, \tau}(\Cdot; \data_k)(\vzero, 0) 
    = \DataDiscr\bigl( \ForwardOp(\template), \data_k \bigr)$
and, in addition, $\Vert \vfield \Vert_2 = 0$ and $\Vert \intenfield \Vert_2 = 0$, so \cref{eq:BoundedVfield} becomes 
\begin{equation}\label{eq:BoundedVfield2}
  \Vert \vfield^k \Vert_2^2 \leq \frac{2}{\gamma} \DataDiscr\bigl( \ForwardOp(\template), \data_k \bigr)
  \to \DataDiscr(\ForwardOp(\template), \data) 
  \quad\text{as $k \to \infty$.}
\end{equation}    
In conclusion, the sequence $( \vfield^k )_k \subset \VelocitySpace{2}{V}$ is bounded.
In a similar way, we can show that $( \intenfield^k )_k \subset \VelocitySpace{2}{\RecSpace}$ is bounded.

The boundedness of both sequences implies that there are sub sequences to these that converge weakly to some 
element $\vfield^{\infty} \in \VelocitySpace{2}{V}$ and $\intenfield^{\infty} \in \VelocitySpace{2}{\RecSpace}$, respectively.  
Thus, to complete the proof, we need to show that $(\vfield^{\infty}, \intenfield^{\infty}) \in \VelocitySpace{2}{V \times \RecSpace}$ minimizes $\ObjFunc_{\gamma, \tau}(\Cdot; \data)$, i.e., that 
\[ \ObjFunc_{\gamma, \tau}(\vfield^{\infty},\intenfield^{\infty}; \data) \leq \ObjFunc_{\gamma, \tau}(\vfield, \intenfield; \data)
   \quad\text{holds for any $(\vfield, \intenfield) \in \VelocitySpace{2}{V \times \RecSpace}$.} 
\]
%seek an upper bound to $\ObjFunc_{\gamma, \tau}(\vfield^{\infty},\intenfield^{\infty}; \data)$.
%\todo{Why would an upper bound imply that $(\vfield^{\infty}, \intenfield^{\infty})$ minimizes $\ObjFunc_{\gamma, \tau}(\Cdot; \data)$? I changed the sentence}
From the weak convergences, we obtain
\begin{multline}\label{eq:Ineq1}
\frac{\gamma}{2} \Vert \vfield^{\infty}\Vert_2^2 + \frac{\tau}{2} \Vert  \intenfield^{\infty} \Vert_2^2 
   \leq \frac{\gamma}{2} \liminf_{k} \Vert \vfield^k \Vert_2^2 + \frac{\tau}{2} \liminf_{k} \Vert \intenfield^k \Vert_2^2 
\\
  \leq \frac{1}{2} \liminf_{k} \Bigl[ \gamma \Vert \vfield^k \Vert_2^2 + \tau \Vert \intenfield^k \Vert_2^2 \Bigr]. 
\end{multline}
%\todo{How do you get the first inequality in \cref{eq:Ineq1}? It is a consequence of the weak convergences, no ?}
The weak convergence also implies (see proof of \cref{prop:Existence}) that 
\[ \DeforOp\bigl(\diffeoflow{k}{0,1}, \templateflow{\infty}{1}\bigr) 
   \rightharpoonup  
   \DeforOp\bigl(\diffeoflow{\infty}{0,1}, \templateflow{\infty}{1}\bigr) 
   \quad\text{in $\RecSpace$.}
\]
In the above, we have used the notational convention introduced in \cref{eq:Notation}.
By the lower semi-continuity of $\DataDiscr$, we get  
\begin{equation}\label{eq:Ineq2}
  \DataDiscr\Bigl(\ForwardOp\bigl(\DeforOp(\diffeoflow{\infty}{0,1},\templateflow{\infty}{1})\bigr),\data \Bigr)
  \leq \liminf_k
  \DataDiscr\Bigl(\ForwardOp\bigl(\DeforOp(\diffeoflow{k}{0,1},\templateflow{k}{1})\bigr),\data_k \Bigr).
\end{equation}

Hence, 
\begin{multline}\label{eq:Ineq3}
\ObjFunc_{\gamma, \tau}(\vfield^{\infty},\intenfield^{\infty}; \data) 
  =  \frac{\gamma}{2} \Vert \vfield^{\infty} \Vert_2^2 + \frac{\tau}{2} \Vert \intenfield^{\infty} \Vert_2^2 
       +  \DataDiscr\Bigl(\ForwardOp\bigl(\DeforOp(\diffeoflow{\infty}{0,1},\templateflow{\infty}{1})\bigr),\data \Bigr).
\\        
   \leq \frac{1}{2} \liminf_{k} \Bigl[ \gamma \Vert \vfield^k \Vert_2^2 + \tau \Vert \intenfield^k \Vert_2^2 \Bigr]    
    + \liminf_{k} \DataDiscr\Bigl(\ForwardOp\bigl(\DeforOp(\diffeoflow{k}{0,1},\templateflow{k}{1})\bigr),\data_k \Bigr)
\\
   \leq  \liminf_{k} \ObjFunc_{\gamma, \tau}(\vfield^k, \intenfield^k; \data_k).
\end{multline}
Next, since $(\vfield^k, \intenfield^k) \in \VelocitySpace{2}{V \times \RecSpace}$ minimizes $\ObjFunc_{\gamma, \tau}(\Cdot; \data_k)$, we get 
\begin{equation*}
\ObjFunc_{\gamma, \tau}(\vfield^{\infty},\intenfield^{\infty}; \data) 
   \leq 
      \liminf_{k} \ObjFunc_{\gamma, \tau}(\vfield, \intenfield; \data_k) 
\end{equation*}
for any $(\vfield, \intenfield) \in \VelocitySpace{2}{V \times \RecSpace}$. 
Furthermore, 
\[ \ObjFunc_{\gamma, \tau}(\vfield, \intenfield; \data_k) \to \ObjFunc_{\gamma, \tau}(\vfield, \intenfield; \data), \]
so  
\[ \ObjFunc_{\gamma, \tau}(\vfield^{\infty},\intenfield^{\infty}; \data) \leq \ObjFunc_{\gamma, \tau}(\vfield, \intenfield; \data)
  \quad\text{for all $(\vfield, \intenfield) \in \VelocitySpace{2}{V \times \RecSpace}$.}
\]
In particular, we have shown that $(\vfield^{\infty},\intenfield^{\infty})$ minimises $\ObjFunc_{\gamma, \tau}(\Cdot; \data)$.
\end{proof}

Our final results concerns convergence, which investigates the behaviour of the solution as data error tends to zero and regularization parameters are adapted accordingly through a parameter choice rule against the data error.
\begin{proposition}[Convergence]\label{prop:Conv}
Let $\data \in \DataSpace$ and assume 
\[ \ForwardOp\bigl(\DeforOp(\diffeoflow{\vfield}{0,1}, \templateflow{\vfield,\intenfield}{1})\bigr) =  \data
    \quad\text{for some $(\vfield, \intenfield) \in \VelocitySpace{2}{V \times \RecSpace}$.}
\]
Next, for parameter choice rules $\delta \mapsto \gamma(\delta)$ and $\delta \mapsto \tau(\delta)$ with $\delta>0$, 
define 
\[
 (\vfield_{\delta},\intenfield_{\delta})
   \in \argmin_{(\vfield, \intenfield)} \ObjFunc_{\gamma(\delta), \tau(\delta)}(\vfield, \intenfield; \data + \datanoise_\delta)
\]
where $\datanoise_\delta \in \DataSpace$ (data error) has magnitude $\Vert \datanoise_\delta \Vert = \delta$. 
Finally, assume that $\delta \mapsto \gamma(\delta)/\tau(\delta)$ and $\delta \mapsto \tau(\delta)/\gamma(\delta)$ are bounded, 
and 
\begin{equation*}
    \lim_{\delta \to 0} \gamma(\delta) 
    =  \lim_{\delta \to 0} \tau(\delta) 
    =  \lim_{\delta \to 0} \frac{\delta^2}{\gamma(\delta)} 
    =  \lim_{\delta \to 0} \frac{\delta^2}{\tau(\delta)} 
    = 0.
\end{equation*}
Then, for any sequence $\delta_k \to 0$ there exists a sub-sequence $\delta_{k'}$ such that $(\vfield_{\delta_{k'}},\intenfield_{\delta_{k'}})$ converges weakly to a $(\vfieldtrue,\intenfieldtrue)$ satisfying $\ForwardOp\bigl(\DeforOp(\diffeoflow{\vfieldtrue}{0,1}, \templateflow{\vfieldtrue,\intenfieldtrue}{1})\bigr) = \data$.
\end{proposition}
\begin{proof}
Let $(\delta_k)$ be a sequence converging to $0$ and, for each $k$, let us denote
\[ \data_k := \data + e_{\delta_k}, \quad \vfield^{k} := \vfield_{\delta_k}, \quad\text{and}\quad
   \intenfield^{k}:= \intenfield_{\delta_k}.
\]
Similarly to previous proofs, we will show that the sequences $(\vfield^k)$ and $(\intenfield^k)$ are bounded, and then that the weakly converging subsequence that can be extracted from $(\vfield^k, \intenfield^k)$ converges to a suitable solution.

Define $\gamma_k := \gamma(\delta_k)$ and $\tau_k := \gamma(\delta_k)$. Then, for each $k$ we have 
\begin{align*}
\Vert  \vfield^k \Vert_2^2 
  &\leq \frac{1}{\gamma_k} \ObjFunc_{\gamma_k, \tau_k, \data_k} (\vfield^k, \intenfield^k) 
    \leq  \frac{1}{\gamma_k} \ObjFunc_{\gamma_k, \tau_k, \data_k} (\widehat{\vfield}, \hat{\intenfield}) 
\\ 
  &=  \frac{1}{\gamma_k}  \Bigl( \gamma_k \Vert \widehat{\vfield}\Vert_2^2 + \tau_k \Vert \hat{\intenfield}\Vert_2^2 + \DataDiscr({\data}, \data_k) \Bigr) 
  \leq  \Vert \widehat{\vfield}\Vert_2^2 + \frac{\tau_k}{\gamma_k} \Vert \hat{\intenfield}\Vert_2^2 + \frac{\delta_k}{\gamma_k}. 
\end{align*}
From the assumptions on the parameter choice rules, we conclude that $(\vfield^k) \subset \VelocitySpace{2}{V}$ is bounded.
Similarly, one can show that $(\intenfield^k)\subset \VelocitySpace{2}{\RecSpace}$ is bounded.

From the above, we conclude that there is a subsequence of $(\vfield^k,\intenfield^k)$ that converges weakly to $(\widetilde{\vfield},\widetilde{\intenfield})$ in $\VelocitySpace{2}{V} \times \VelocitySpace{2}{V}$. Then (see proof of \cref{prop:Existence})
\[ \DataDiscr\Bigl(\ForwardOp\bigl(\DeforOp(\diffeoflow{\widetilde{\vfield}}{0,1}, \templateflow{\widetilde{\vfield},\widetilde{\intenfield}}{1}) \bigr), \data \Bigr) 
   \leq 
   \liminf_k \DataDiscr\Bigl(\ForwardOp(\DeforOp\bigl(\diffeoflow{\vfield_k}{0,1},\templateflow{\widetilde{\vfield},\widetilde{\intenfield}}{1}) \bigr), \data_k \Bigr).
\]
Furthermore, the above quantity converges to $0$ since 
\begin{align*}
\DataDiscr\Bigl(\ForwardOp\bigl(\DeforOp(\diffeoflow{\vfield^k}{0,1}, \templateflow{\vfield^k,\intenfield^k}{1})\bigr), \data_k \Bigr) 
&\leq \ObjFunc_{\gamma_k, \tau_k, \data_k} (\vfield^k, \intenfield^k) 
\leq \ObjFunc_{\gamma_k, \tau_k, \data_k} (\widehat{\vfield}, \hat{\intenfield}) 
\\
&= \gamma_k \Vert \widehat{\vfield}\Vert_2^2 + \tau_k \Vert \hat{\intenfield}\Vert_2^2 + \DataDiscr(\data,\data_k) 
\to 0 \quad\text{and $k \to \infty$.}
\end{align*}
Hence, $\ForwardOp\bigl(\DeforOp(\diffeoflow{\widetilde{\vfield}}{0,1}, \templateflow{\widetilde{\vfield},\widetilde{\intenfield}}{1})\bigr) = \data$.
\end{proof}

\subsection{Numerical implementation}\label{sec:Numericalmpl}
%\todo[inline]{Barbara, give a very short explanation of the numerical implementation for solving \cref{eq.MetaOptim} along with a pointer to the ODL code.}

In order to solve \cref{eq.MetaOptim}, we use a gradient descent scheme on the variable $(\vfield, \intenfield) \in \VelocitySpace{2}{V \times \RecSpace}$ with a uniform discretization of the interval $[0, 1]$ into $N$ parts, i.e., $t_i = 1/N$ for $i=0, \ldots, N$ and the gradient descent is performed on $\vfield (t_i, \Cdot)$, $\intenfield (t_i, \Cdot)$, for $i = 0, 1, \ldots, N$. An alternative approach developed in \cite{neumayer2018regularization} extends the time discrete path method in \cite{EfRuSc18} to the indirect setting. 

In order to compute numerical integrations, we use a Euler scheme on this discretization. The flow equation \eqref{Eq:FlowEq} is computed using the following approximation with small deformations: $\diffeoflow{\vfield}{t_i, 0} \approx \diffeoflow{\vfield}{t_{i-1}, 0} \circ \Big( \Id - \frac{1}{N} \vfield(t_{i-1},\Cdot) \Big)$. 

\Cref{alg:GradientComputation} presents the implementation\footnote{%
\href{https://github.com/bgris/odl/tree/IndirectMatching/examples/Metamorphosis}{https://github.com/bgris/odl/tree/IndirectMatching/examples/Metamorphosis}} for computing the gradient of $\ObjFunc$ and it is based on expressions from \cref{Appendix:GradientComputation}). The computation of the Jacobian determinant $\Bigl\vert \Det(\der{ \diffeoflow{\vfield}{t_i, 1}}(x)) \Bigr\vert$ at each time point is based on the following approximation similar to \cite{Chen:2018aa}: 
\[ \Bigl\vert \Det\bigl(\der{ \diffeoflow{\vfield}{t_i, 1}}(x)\bigr) \Bigr\vert 
      \approx \Bigl(1 + \frac{1}{N} \Div\vfield (t_i,\Cdot)\Bigr)\bigl\vert \D\diffeoflow{\vfield }{t_{i+1},1}\bigr\vert \circ 
      \Bigl(\Id  + \frac{1}{N} \vfield (t_i,\Cdot)\Bigr).
\]

\begin{algorithm}
\caption{Computation of $\nabla \ObjFunc (\vfield, \intenfield)$.}\label{alg:GradientComputation}
{\small
\begin{algorithmic}[1]
  \Require 
    $\vfield (t_i,\Cdot)$ and $\intenfield (t_i,\Cdot)$ with $t_i \gets i/N$ for $i = 0, 1, \ldots, N$.
  \For{ $i = 1, \ldots, N$}
    \Comment{Compute $\intenfield  (t_i,\Cdot) \circ \diffeoflow{\vfield }{0, t_i}$}
    %\State $\intenfield  (t_i,\Cdot) \circ \diffeoflow{\vfield }{0, t_0} \gets \intenfield  (t_i,\Cdot)$
    \State $temp \gets \intenfield  (t_i,\Cdot)$
    \For{$j = i-1, \ldots, 0$}
%          \State 
%         $\intenfield  (t_i,\Cdot) \circ \diffeoflow{\vfield }{0, t_{j+1}} 
%           \gets \bigl(\intenfield  (t_i,\Cdot) \circ \diffeoflow{\vfield }{0, t_j} \bigr) \circ \Bigl(\Id + \frac{1}{N}\vfield (t_j,\Cdot)\Bigr)$
      \State 
         $temp
           \gets temp \circ \Bigl(\Id + \frac{1}{N}\vfield (t_j,\Cdot)\Bigr)$
    \EndFor
    \State 
         $\intenfield  (t_i,\Cdot) \circ \diffeoflow{\vfield }{0, t_{i}}  \gets temp $
  \EndFor
  \For{ $i = 1, \ldots, N$} 
    \Comment{Compute $\image^{\vfield , \intenfield } (t_i,\Cdot) := \imagetemp^{\vfield ,\intenfield }(t_i,\Cdot) \circ \diffeoflow{\vfield }{t_i, 0}$}
    \State 
      $\imagetemp^{\vfield ,\intenfield }(t_i,\Cdot) 
        \gets I_0 + \sum_{j=0}^{i-1} \imagetemp^{\vfield ,\intenfield }(t_{j},\Cdot) 
          + \frac1N \intenfield  (t_{j},\Cdot) \circ \diffeoflow{\vfield }{0, t_{j}}$ 
    \State 
      $\imagetemp^{\vfield ,\intenfield } (t_i,\Cdot) \circ \diffeoflow{\vfield }{0, 0} 
        \gets \imagetemp^{\vfield ,\intenfield } (t_i,\Cdot)$
    \For{$j = 1, \ldots, i$}
     \State
       $\imagetemp^{\vfield ,\intenfield } (t_i,\Cdot) \circ \diffeoflow{\vfield }{ t_{j}, 0 } 
         \gets \bigl( \imagetemp^{\vfield ,\intenfield }(t_i,\Cdot) \circ \diffeoflow{\vfield }{0, t_{j-1}} \bigr) \circ \Bigl(\Id  - \frac{1}{N}\vfield (t_{j-1},\Cdot)\Bigr)$
    \EndFor
  \EndFor
  \For{$i = 1, \ldots, N$} 
    \Comment{Compute $\template \circ \diffeoflow{\vfield }{t_i,0}$}
    \State 
      $\template \circ \diffeoflow{\vfield }{0,0} 
        \gets \template \circ \diffeoflow{\vfield }{0,0} = \template$
    \State 
      $\template \circ \diffeoflow{\vfield }{t_i,0} 
        \gets \bigl(\template \circ \diffeoflow{\vfield }{t_{i-1},0}\bigr) \circ \Bigl(\Id - \frac{1}{N}\vfield (t_{i-1},\Cdot)\Bigr)$
  \EndFor
  \For{$i = 1, \ldots, N$}
    \State 
      $G(t_i, \Cdot) 
        \gets  \nabla(\template \circ \diffeoflow{\vfield}{t_i,0})   + 
          \sum_{j=0}^{t_{i-1}} \dfrac1N  \nabla(\zeta( t_j, \cdot ) \circ \diffeoflow{\vfield}{t_i, t_j} )  $
  \EndFor
  \State  $\bigl\vert \D\diffeoflow{\vfield }{t_{N},1}\bigr\vert = \bigl\vert \D\diffeoflow{\vfield }{1,1}\bigr\vert = 1$ 
  \Comment{Compute $\bigl\vert \D\diffeoflow{\vfield }{t_i,1}\bigr\vert$}
  \For{$i = N-1, \ldots, 0$}
    \State 
      $\bigl\vert \D\diffeoflow{\vfield }{t_i,1}\bigr\vert 
        \gets \Bigl(1 + \frac{1}{N} \Div\vfield (t_i,\Cdot)\Bigr)\bigl\vert \D\diffeoflow{\vfield }{t_{i+1},1}\bigr\vert \circ \Bigl(\Id  + \frac{1}{N} \vfield (t_i,\Cdot)\Bigr)$
  \EndFor
  \State 
    $\nabla\DataDiscr\bigl( \image^{\vfield,\intenfield}(1, \cdot), \data \bigr)\bigl(\diffeoflow{\vfield }{t_N, 1}\bigr)
      \gets \nabla\DataDiscr\bigl( \image^{\vfield,\intenfield}(1, \cdot), \data \bigr)$ 
  \For{$i = N-1, \ldots, 0$}
    \Comment{Compute $ \nabla\DataDiscr\bigl( \image^{\vfield,\intenfield}(1, \cdot), \data \bigr)\bigl(\diffeoflow{\vfield }{t_i, 1}\bigr)$}
    \State 
      $\nabla\DataDiscr\bigl( \image^{\vfield,\intenfield}(1, \cdot), \data \bigr)\bigl(\diffeoflow{\vfield }{t_i, 1}\bigr) 
        \gets  \nabla\DataDiscr\bigl( \image^{\vfield,\intenfield}(1, \cdot), \data \bigr)\bigl(\diffeoflow{\vfield }{t_{i+1}, 1}\bigr) 
          \circ \Bigl(\Id  + \frac{1}{N} \vfield (t_i,\Cdot)\Bigr)$
  \EndFor
  \For{$i = 1, \ldots, N$} 
    \Comment{Compute $\nabla \ObjFunc (\vfield , \intenfield )$} %using FFT techniques for computing the convolution with the kernel}
    \State \begin{multline*}
        {\nabla_{\vfield}} \ObjFunc_{\gamma, \tau} (\vfield , \intenfield , \data)(t_i,,\Cdot) 
        \gets  2 \gamma \vfield (t_i,,\Cdot)  \\
         -  \int_{\domain}  K(x, \cdot ) 
        \Bigl\vert \Det(\der{ \diffeoflow{\vfield}{t_i, 1}}(x)) \Bigr\vert   
        \nabla\DataDiscr\bigl( \image^{\vfield,\intenfield}(1, \cdot), \data \bigr)\bigl(\diffeoflow{\vfield }{t_i, 1}(x)\bigr) G(t_i, x) \der x
      \end{multline*}
    \State
      \begin{multline*}
        \nabla_\intenfield \ObjFunc_{\gamma, \tau} (\vfield , \intenfield )(t_i,,\Cdot) 
        \gets  2 \tau \intenfield (t_i,,\Cdot) \\
           + \vert \Det(\der{ \diffeoflow{\vfield}{t_i, 1}}) \Bigr\vert  \nabla\DataDiscr\bigl( \image^{\vfield,\intenfield}(1, \cdot), \data \bigr)\bigl(\diffeoflow{\vfield }{t_i, 1}(x)\bigr) G(t_i, \Cdot) 
      \end{multline*}
        \EndFor
\State \Return $\nabla \ObjFunc (\vfield )(t_i,\Cdot)$, $\nabla \ObjFunc (\intenfield )(t_i,\Cdot)$ for $i = 1, \ldots, N$.
\end{algorithmic}
}
\end{algorithm}

\section{Application to 2D tomography}\label{sec:TomoExample}
\subsection{The setting}\label{sec:TomoInvProb}
\paragraph{The forward operator}
Let $\RecSpace=L^{2}(\domain,\Real)$ whose elements represent 2D images on a fixed bounded domain $\domain \subset \Real^2$. 
In the application shown here, diffeomorphisms act on $\RecSpace$ through a geometric group action in \cref{eq:GeometricGroupAction} 
and the goal is to register a given differentiable template image $\template \in \RecSpace$ against a target that observed indirectly as in \cref{eq:InvProb}. 

The forward operator $\ForwardOp \colon \RecSpace \to \DataSpace$ in 2D tomographic imaging is the 2D ray/Radon transform, i.e.,
\[ \ForwardOp(\image)(\omega,x) = \int_{\Real} \image(x+s \omega) \der s
    \quad\text{for $\omega \in S^1$ and $x \in \omega^{\bot}$.}
\]  
Here, $S^1$ is the unit circle, so $(\omega,x)$ encodes the line $s \mapsto x + s \omega$ in $\Real^2$ with direction $\omega$ through $x$. 
The data manifold $\datadomain$ is the set of such lines that are included in the measurements, i.e., $\datadomain$ is given by the experimental set-up.
We will consider parallel lines in $\Real^2$ (parallel beam data), i.e., tomographic data are noisy digitized values of an $L^2$-function on this manifold so $\data \in \DataSpace = L^{2}(\datadomain,\Real)$. The forward operator is linear, so it is particular Gateaux differentiable, and the adjoint of its derivative is given by the backprojection, see \cite{NaWu01,Markoe:2006aa} for further details.

If data is corrupted by additive Gaussian noise, so a suitable data likelihood is the 2-norm, i.e., 
\[
  \DataDiscr \colon \DataSpace \times \DataSpace \to \Real
  \quad\text{with}\quad
  \DataDiscr(\data,\dataother) = \Vert \data - \dataother \Vert_{2}^2.
\]
The noise level in data is specified by the \ac{PSNR}, which is defined as 
\[ \SNR(\data) = 10 \log_{10} \biggl( \frac{\Vert \data_0 - \overline{\data_0} \Vert^2}{ \Vert \datanoise - \overline{\datanoise}  \Vert^2} \biggr) 
   \quad\text{for $\data = \data_0 + \datanoise$.}
\]
In the above, $\data_0$ is the noise-free part and $\datanoise$ is the noise component of data with $\overline{\data_0}$ and $\overline{\datanoise}$ denoting the mean of $\data_0$ and $\datanoise$, respectively. The \ac{PSNR} is expressed in terms of {\decibel}.

\paragraph{Joint tomographic reconstruction and registration}
Under the geometric group action \cref{eq:GeometricGroupAction}, metamorphosis based-indirect registration reads as 
 \[
 \image^{\vfieldopt,\intenfieldopt}_1 = \DeforOp\bigl(\diffeoflow{\vfieldopt}{0,1}, \templateflow{\vfieldopt,\intenfieldopt}{1}\bigr)
   = \templateflow{\vfieldopt,\intenfieldopt}{1} \circ \diffeoflow{\vfieldopt}{1,0}
 \]
where $(\vfieldopt,\intenfieldopt) \in \VelocitySpace{2}{V \times \RecSpace}$ minimizes \cref{Eq:ObjectiveFunctional}, i.e., 
given regularization parameters $\gamma, \tau \geq 0$ and initial template $\template \in \RecSpace$ we solve 
\begin{equation}\label{eq:OptimIndirRegTomo}
\begin{split}
 \displaystyle{\min_{(\vfield, \intenfield)}} &
\biggl[ \frac{\gamma}{2} \Vert \vfield \Vert_2^2 + \frac{\tau}{2} \Vert \intenfield \Vert_2^2 
        + \Bigl\Vert \ForwardOp\Bigl( \image\bigl(1,\diffeo(1,\Cdot)^{-1} \bigr)\Bigr) - \data \Bigr\Vert_2^2
\biggr] 
\\[1.5em]
& 
\begin{cases}       
\dfrac{\der}{\der t} \image(t,x) = \intenfield\bigl(t, \diffeo(t,x)\bigr) & \\[0.75em]
\image(0,x) = \template(x) & \\[1em]
\displaystyle{\frac{d}{dt}} \diffeo(t,x) = \vfield\bigl( t, \diffeo(t,x) \bigr)  & \\[0.5em] 
\diffeo(0,x)= x. &
\end{cases}  
\end{split}
\end{equation}
We will consider a set $V$ of vector fields that is an \ac{RKHS} with a reproducing kernel represented by symmetric and positive definite 
Gaussian. Then $V$ is admissible and is continuously embedded in $L^2(\domain,\Real^2)$. The kernel we pick is 
$K_\sigma \colon \domain \times \domain \to \Real_{+}^{2 \times 2}$ 
\begin{equation}\label{eq:KernelEq}
  K_\sigma(x,y) := 
      \exp\Bigl(-\dfrac{1}{2 \sigma^2} \Vert x-y \Vert_2 \Bigr)
      \begin{pmatrix} 
          1  & 0 \\
          0  & 1
      \end{pmatrix} 
\quad\text{for $x,y \in \Real^2$ and $\sigma >0$.}
\end{equation}
The kernel-size $\sigma$ also acts as a regularization parameter.

\subsection{Overview of experiments}
In the following we perform a number of experiments that tests various aspects of using metamorphoses based indirect registration for joint tomographic reconstruction and registration. 
The tomographic inverse problem along with characteristics of the data are outlined in  \cref{sec:TomoInvProb}. 
The results are obtained by solving \cref{eq:OptimIndirRegTomo} via a gradient descent, see \cref{Appendix:GradientComputation} for the computation of 
the gradient of the objective. 
For each reconstruction, we list the the number of angles of the parallel beam ray transform, the kernel-size $\sigma$ in \cref{eq:KernelEq}, and the two regularisation parameters $\gamma, \tau >0$ appearing in the objective functional in \cref{eq:OptimIndirRegTomo}.

The first test (\cref{Sec:Test:BadIntensity}) aims to show how metamorphoses based indirect registration handles a template that has intensities differing from those of the target. 
\Cref{Sec:SheppLogan} considers the ability to handle an initial template with a topology that does not match the target. 
This is essential when one has simultaneous geometric and topological changes. 
As an example, in spatiotemporal imaging it may very well be the case that geometric deformation takes place simultaneously as new masses appear or disappear. 
Next, in \cref{Sec:Test:Robustness} studies the robustness of the solutions with respect to variations in the regularization parameters.
Finally, \cref{Sec:Test:SpatioTemporal} shows how indirect registration through metamorphoses can be used to recover a temporal evolution of a given template registered against time series of data. This is an essential part of spatio-temporal tomographic reconstruction. 

\Cref{Sec:Test:BadIntensity,Sec:SheppLogan,Sec:Test:Robustness} have a common setting in that grey scale images in the reconstruction space are discretised using $256 \times 256$ pixels in the image domain $\Omega = [-16, 16] \times [-16, 16]$. The tomographic data is noisy samples of the 2D parallel beam ray transform of the target sampled at 100 angles uniformly distributed angles in $[0,\pi]$ with $362$ lines/angle. Data is corrupted with additive Gaussian noise with differing noise levels.

\subsection{Consistent topology and inconsistent intensities}\label{Sec:Test:BadIntensity}
Here, topology of the template is consistent with that of the target, but intensities differ. The template, which is shown in \cref{Fig:Metamorphosis:Triangles:Template}, is registered against tomographic data shown in \cref{Fig:Metamorphosis:Triangles:Data}.
The (unknown) target used to generate data is shown in \cref{Fig:Metamorphosis:Triangles:Target}. 
Also, data has a noise level corresponding to a \ac{PSNR} of $\unit{15.6}{\decibel}$ and kernel size is $\sigma = 2$, which should be compared to the size of the image domain $\Omega = [-16, 16] \times [-16, 16]$.
The final reconstruction is shown in \cref{Fig:Metamorphosis:Triangles:Reconstruction}, which is to be compared against the target in
\cref{Fig:Metamorphosis:Triangles:Target}. \Cref{Fig:Metamorphosis:Triangles} also shows image, deformation and template trajectories.

We clearly see that metamorphosis based indirect registration can handle a template with wrong intensities. As a comparison, see \cref{Fig:LDDMM:BadIntensities:Reco} for the corresponding \ac{LDDMM} based indirect registration using the same template and data. Furthermore, the different trajectories also provides easy visual interpretation of the influence of geometric and intensity deformations. 

\begin{figure}[t]
 \centering
% Row 1
  \subfloat[Template.\label{Fig:Metamorphosis:Triangles:Template}]{\includegraphics[width=0.28\textwidth]{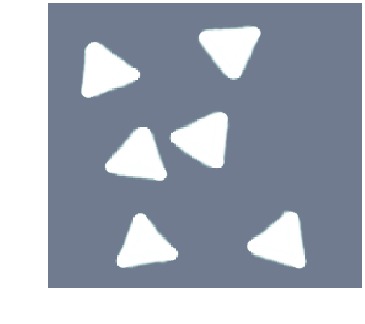}}%
  \hfil
  \subfloat[Target.\label{Fig:Metamorphosis:Triangles:Target}]{\includegraphics[width=0.28\textwidth]{test1ground_truth_values__0_2__0_9}}%
  \hfil
  \subfloat[Data (sinogram).\label{Fig:Metamorphosis:Triangles:Data}]{\includegraphics[width=0.28\textwidth]{test1ground_truth_values__0__1sinogram.jpeg}}%
\par% Row 2
 \subfloat[$t = 0$.\label{Fig:Metamorphosis:Triangles:ImageStart}]{\rotatebox{90}{\hspace{0.7cm}\footnotesize{Image}\hspace{-0cm}}%
 \includegraphics[width=0.19\textwidth]{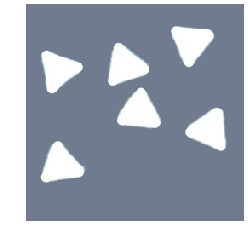}}%
  \hfil
  \subfloat[$t = 0.2$.]{\includegraphics[width=0.19\textwidth]{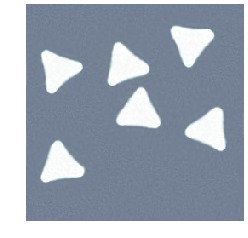}}%
  \hfil
  \subfloat[$t = 0.5$.]{\includegraphics[width=0.19\textwidth]{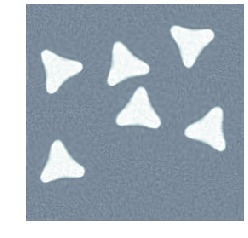}}%
  \hfil
  \subfloat[$t = 0.7$.]{\includegraphics[width=0.19\textwidth]{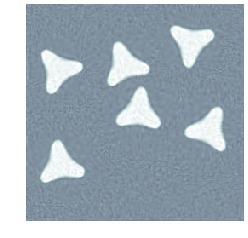}}%
  \hfil
  \subfloat[$t = 1$.\label{Fig:Metamorphosis:Triangles:Reconstruction}]{\includegraphics[width=0.19\textwidth]{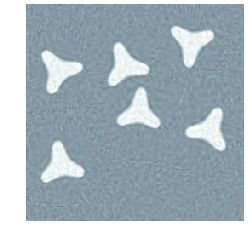}}%
\par% Row 3
  \subfloat[$t = 0$.\label{Fig:Metamorphosis:Triangles:DeforStart}]{\rotatebox{90}{\hspace{0.25cm}\footnotesize{Deformation}\hspace{-0cm}}%
  \includegraphics[width=0.19\textwidth]{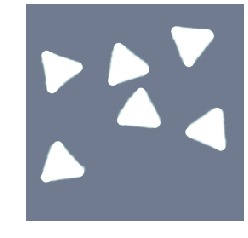}}%
  \hfil
  \subfloat[$t = 0.2$.]{\includegraphics[width=0.19\textwidth]{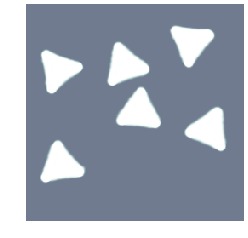}}%
  \hfil
  \subfloat[$t = 0.5$.]{\includegraphics[width=0.19\textwidth]{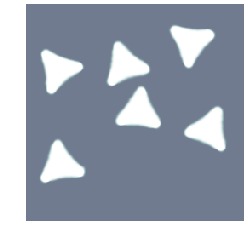}}%
  \hfil
  \subfloat[$t = 0.7$.]{\includegraphics[width=0.19\textwidth]{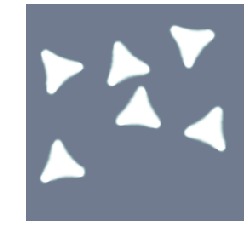}}%
  \hfil
  \subfloat[$t = 1$.\label{Fig:Metamorphosis:Triangles:DeforEnd}]{\includegraphics[width=0.19\textwidth]{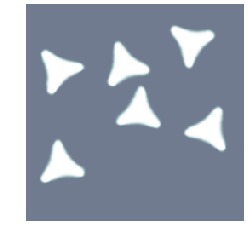}}%
\par% Row 4
  \subfloat[$t = 0$.\label{Fig:Metamorphosis:Triangles:TemplateStart}]{\rotatebox{90}{\hspace{0.5cm}\footnotesize{Intensity}\hspace{-0cm}}%
  \includegraphics[width=0.19\textwidth]{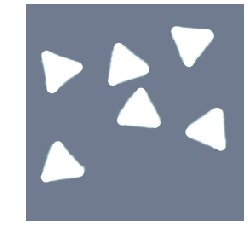}}%
  \hfil
  \subfloat[$t = 0.2$.]{\includegraphics[width=0.19\textwidth]{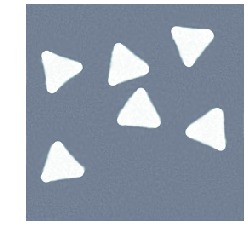}}%
  \hfil
  \subfloat[$t = 0.5$.]{\includegraphics[width=0.19\textwidth]{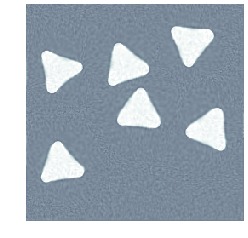}}%
  \hfil
  \subfloat[$t = 0.7$.]{\includegraphics[width=0.19\textwidth]{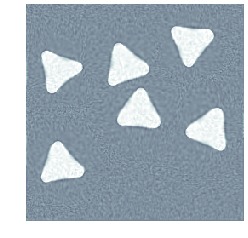}}%
  \hfil
  \subfloat[$t = 1$.\label{Fig:Metamorphosis:Triangles:TemplateEnd}]{\includegraphics[width=0.19\textwidth]{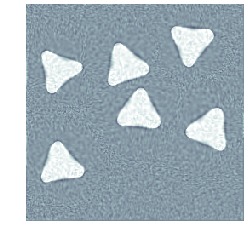}}%
% Caption 
  \caption{\label{Fig:Metamorphosis:Triangles}  
  Metamorphosis based indirect-matching of template in \captionsubref{Fig:Metamorphosis:Triangles:Template} against data in 
  \captionsubref{Fig:Metamorphosis:Triangles:Data}, which represents 2D ray transform of target in \captionsubref{Fig:Metamorphosis:Triangles:Target}
(100 uniformly distributed angles in $[0,\pi]$). 
  The second row \captionsubref{Fig:Metamorphosis:Triangles:ImageStart}--\captionsubref{Fig:Metamorphosis:Triangles:Reconstruction} shows the image trajectory $t \mapsto \DeforOp(\diffeoflow{\vfield}{0,t}, \image_t(\vfield,\intenfield))$, so the final registered template is in \captionsubref{Fig:Metamorphosis:Triangles:Reconstruction}.
  The third row \captionsubref{Fig:Metamorphosis:Triangles:DeforStart}--\captionsubref{Fig:Metamorphosis:Triangles:DeforEnd} shows the deformation trajectory $t \mapsto \DeforOp(\diffeoflow{\vfield}{0,t}, \template)$, likewise the fourth row \captionsubref{Fig:Metamorphosis:Triangles:TemplateStart}--\captionsubref{Fig:Metamorphosis:Triangles:TemplateEnd} shows the intensity trajectory $t \mapsto \image_t(\vfield,\intenfield)$.} 
\end{figure}

\subsection{Inconsistent topology and intensities}\label{Sec:SheppLogan}
Here, both topology and intensities of the template differ from those in the target.
The template, which is shown in \cref{Fig:Metamorphosis:SheppLogan:Template}, is registered against tomographic data shown in \cref{Fig:Metamorphosis:SheppLogan:Data}. 
The (unknown) target used for generating the data is shown in \cref{Fig:Metamorphosis:SheppLogan:Target}. 
Also, data has a noise level corresponding to a \ac{PSNR} of $\unit{10.6}{\decibel}$ and kernel size is $\sigma = 2$, which should be compared to the size of the image domain $\Omega = [-16, 16] \times [-16, 16]$.
The final reconstruction is shown in \cref{Fig:Metamorphosis:SheppLogan:Reconstruction}, which is to be compared against the target in
\cref{Fig:Metamorphosis:SheppLogan:Target}. \Cref{Fig:Metamorphosis:SheppLogan} also shows image, deformation and template trajectories.

We clearly see that metamorphosis based indirect registration can handle a template where both intensities and the topology are wrong. 
In particular, we can see follow both the deformation of the template and the appearance of the white disc.

\begin{figure}
\centering
% Row 1
  \subfloat[Template $\template$.\label{Fig:Metamorphosis:SheppLogan:Template}]{\includegraphics[width=0.28\textwidth]{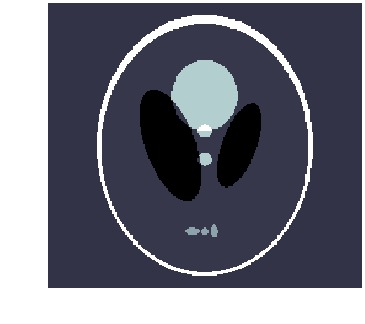}}%
  \hfil
  \subfloat[Target.\label{Fig:Metamorphosis:SheppLogan:Target}]{\includegraphics[width=0.28\textwidth]{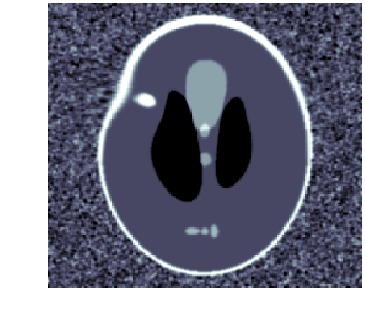}}%
  \hfil
  \subfloat[Data (sinogram).\label{Fig:Metamorphosis:SheppLogan:Data}]{\includegraphics[width=0.28\textwidth]{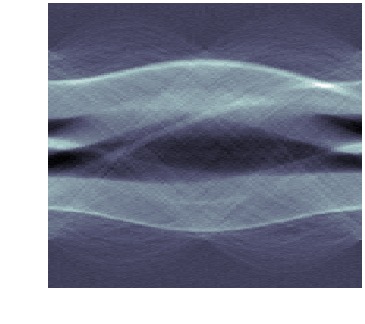}}\par% Row 2
  \subfloat[$t = 0$.\label{Fig:Metamorphosis:SheppLogan:ImageStart}]{\rotatebox{90}{\hspace{0.7cm}\footnotesize{Image}\hspace{-0cm}}%
  \includegraphics[width=0.19\textwidth]{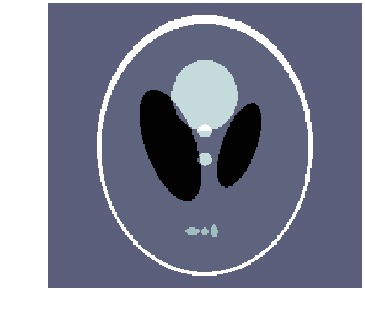}}%
  \hfil
  \subfloat[$t = 0.2$.]{\includegraphics[width=0.19\textwidth]{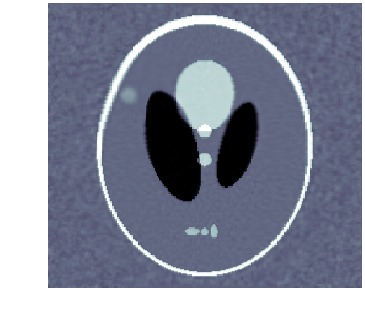}}%
  \hfil
  \subfloat[$t = 0.5$.]{\includegraphics[width=0.19\textwidth]{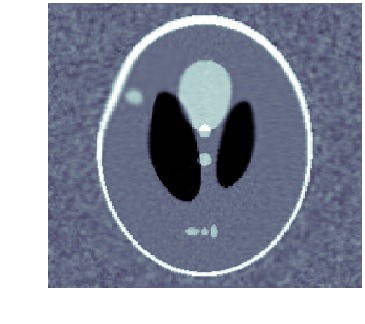}}%
  \hfil
  \subfloat[$t = 0.7$.]{\includegraphics[width=0.19\textwidth]{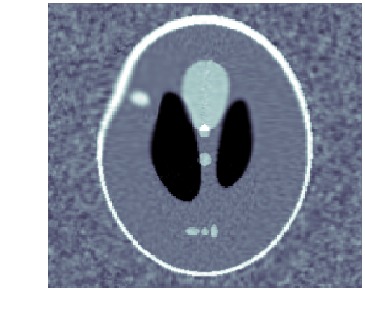}}%
  \hfil
  \subfloat[$t = 1$.\label{Fig:Metamorphosis:SheppLogan:Reconstruction}]{\includegraphics[width=0.19\textwidth]{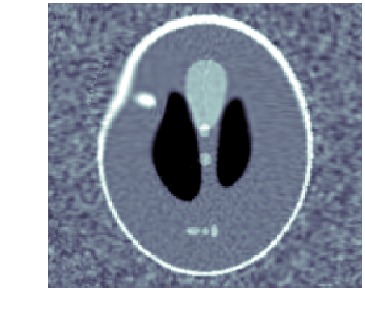}}%  
\par% Row 3
    \subfloat[$t = 0$.\label{Fig:Metamorphosis:SheppLogan:DeforStart}]{\rotatebox{90}{\hspace{0.25cm}\footnotesize{Deformation}\hspace{-0cm}}%
    \includegraphics[width=0.19\textwidth]{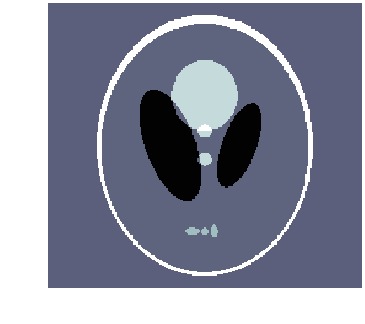}}%
  \hfil
  \subfloat[$t = 0.2$.]{\includegraphics[width=0.19\textwidth]{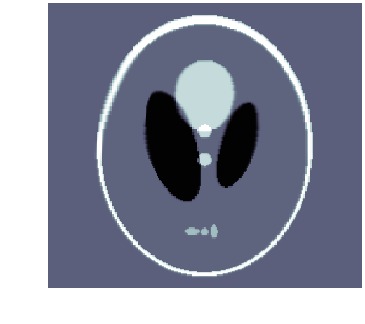}}%
  \hfil
  \subfloat[$t = 0.5$.]{\includegraphics[width=0.19\textwidth]{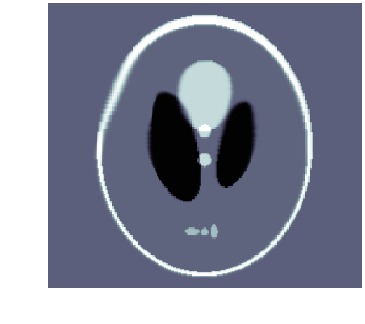}}%
  \hfil
  \subfloat[$t = 0.7$.]{\includegraphics[width=0.19\textwidth]{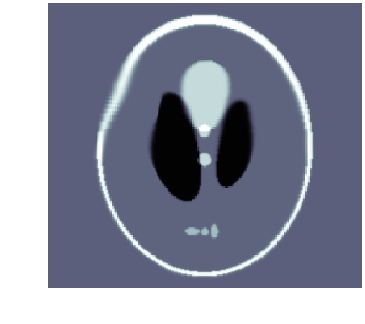}}%
  \hfil
  \subfloat[$t = 1$.\label{Fig:Metamorphosis:SheppLogan:DeforEnd}]{\includegraphics[width=0.19\textwidth]{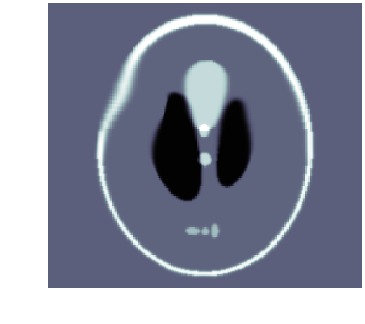}}%
\par% Row 4
   \subfloat[$t = 0$.\label{Fig:Metamorphosis:SheppLogan:TemplateStart}]{\rotatebox{90}{\hspace{0.5cm}\footnotesize{Intensity}\hspace{-0cm}}%
   \includegraphics[width=0.19\textwidth]{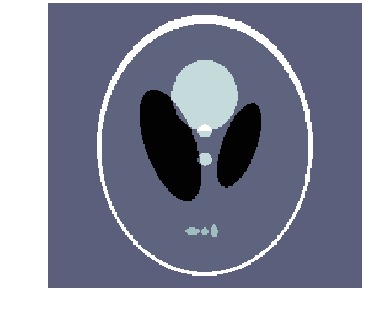}}%
  \hfil
  \subfloat[$t = 0.2$.]{\includegraphics[width=0.19\textwidth]{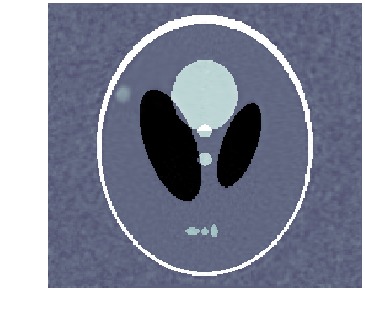}}%
  \hfil
  \subfloat[$t = 0.5$.]{\includegraphics[width=0.19\textwidth]{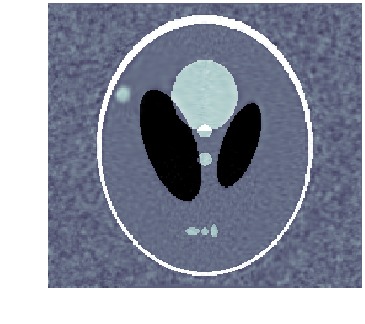}}%
  \hfil
  \subfloat[$t = 0.7$.]{\includegraphics[width=0.19\textwidth]{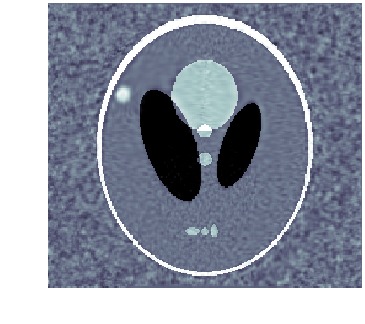}}%
  \hfil
  \subfloat[$t = 1$.\label{Fig:Metamorphosis:SheppLogan:TemplateEnd}]{\includegraphics[width=0.19\textwidth]{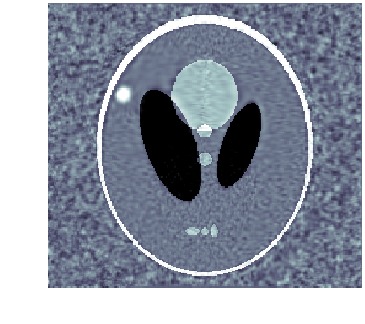}}%
% Caption 
 \caption{\label{Fig:Metamorphosis:SheppLogan}
   Metamorphosis based indirect-matching of template in \captionsubref{Fig:Metamorphosis:Triangles:Template} against data in 
  \captionsubref{Fig:Metamorphosis:SheppLogan:Data}, which represents 2D ray transform of target in \captionsubref{Fig:Metamorphosis:SheppLogan:Target}
(100 uniformly distributed angles in $[0,\pi]$). 
  The second row \captionsubref{Fig:Metamorphosis:SheppLogan:ImageStart}--\captionsubref{Fig:Metamorphosis:SheppLogan:Reconstruction} shows the image trajectory $t \mapsto \DeforOp(\diffeoflow{\vfield}{0,t}, \image_t(\vfield,\intenfield))$, so the final registered template is in \captionsubref{Fig:Metamorphosis:SheppLogan:Reconstruction}.
  The third row \captionsubref{Fig:Metamorphosis:SheppLogan:DeforStart}--\captionsubref{Fig:Metamorphosis:SheppLogan:DeforEnd} shows the deformation trajectory $t \mapsto \DeforOp(\diffeoflow{\vfield}{0,t}, \template)$, likewise the fourth row \captionsubref{Fig:Metamorphosis:SheppLogan:TemplateStart}--\captionsubref{Fig:Metamorphosis:SheppLogan:TemplateEnd} shows the intensity trajectory $t \mapsto \image_t(\vfield,\intenfield)$.
  }%
\end{figure}

\subsection{Robustness}\label{Sec:Test:Robustness}
Metamorphosis based indirect registration, which amounts to solving \cref{eq:OptimIndirRegTomo}, requires selecting three parameters: the kernel-size $\sigma$ and the two regularisation parameters $\gamma$ and $\tau$. Here we study the influence of these parameters on the final registered image (reconstruction) based on the setup in \cref{Sec:SheppLogan}.

The reconstruction along with its template and deformation parts are not that sensitive to the specific choice the two regularisation parameters $\gamma$ and $\tau$, see \cref{Table:FOM:RegParam} that shows the \ac{SSIM} and \ac{PSNR} values for various values of $\gamma$ and $\tau$ when $\sigma=3$.
The reconstruction is on the other hand more sensitive to the choice of the kernel size, see \cref{Table:FOM:Sigma} for a table of \ac{SSIM} and \ac{PSNR} values corresponding to different choices of kernel size. \Cref{Fig:VariousSigma} also shows reconstructed image with the corresponding final template and deformation parts for various values of $\sigma$.
Interestingly, even if the reconstruction is satisfying for the various values of the kernel size $\sigma$, its template part and deformation parts are really different. The geometric deformation and the change in intensity values seem to balance in an non-intuitive way in order
to produce a reasonable final image. 

\begin{table}
\centering
\begin{tabular}{c | r r r r}
\diagbox{$\tau$}{$\gamma$}  &  \multicolumn{1}{c}{$10^{-7}$} &  \multicolumn{1}{c}{$10^{-5}$} & \multicolumn{1}{c}{$10^{-3}$} 
 &  \multicolumn{1}{c}{$10^{-1}$}   \\ 
\toprule
\multirow{2}{*}{$10^{-1}$}     & 0.767 & 0.768& 0.768 &  0.768 \\
   & -6.37 & -6.43 & -6.422 &  -6.42\\
  \hline
\multirow{2}{*}{$10^{-3}$}   &  0.766  &  0.770 & 0.770&  0.770 \\
  & -6.33 & -6.36  &  -6.36 &  -6.36 \\
  \hline
\multirow{2}{*}{$10^{-5}$}   & 0.766& 0.770  &0.770   & 0.770 \\
  & -6.33 & -6.35  &  -6.36 &  -6.36 \\
  \hline
\multirow{2}{*}{$10^{-7}$ }  &  0.766 &0.770 & 0.770   & 0.770  \\
  &-6.33  & -6.35  & -6.36  & -6.36  \\
\bottomrule
\end{tabular}
\caption{\Ac{SSIM} (top) and \ac{PSNR} (bottom) values for metamorphosis based indirect registration with varying regularisation parameter and $\sigma = 3$ for several regularisation parameters.  \label{Table:FOM:RegParam}}
\end{table}

\begin{table}
\centering
\begin{tabular}{r |  r r r r r r r}
 $\sigma$ & $0.3$  & $0.6$ & $1$ & $2$ & $3$ & $5$ & $10$ \\
\toprule
  \ac{SSIM} & 0.660 & 0.703 & 0.737 & 0.769 & 0.766 & 0.764 & 0.682 \\
  \ac{PSNR} & -7.75 & -7.03 & -6.57 & -6.36 & -6.49 & -6.66 & -8.98 \\
\end{tabular}
\caption{\ac{SSIM} and \ac{PSNR} values for metamorphosis based indirect registration with varying kernel size $\sigma$ and fixed regularisation parameters $\gamma = \tau = 10^{-5}$.  \label{Table:FOM:Sigma}}
\end{table}

\begin{figure}
\vskip-6\baselineskip
\begin{minipage}{0.03\textwidth}

\end{minipage}
\begin{minipage}{0.3\textwidth}
\centering 
Image
\end{minipage}
\begin{minipage}{0.3\textwidth}
\centering 
Deformation 
\end{minipage}
\begin{minipage}{0.3\textwidth}
\centering 
Template 
\end{minipage}

\begin{minipage}{0.03\textwidth}
\rotatebox{90}{\hspace{2cm}$\sigma = 0.3$\hspace{0cm}}
\par
\rotatebox{90}{\hspace{2cm}$\sigma = 0.6$}
\par
\rotatebox{90}{\hspace{2cm}$\sigma = 1$}
\par
\rotatebox{90}{\hspace{2cm}$\sigma = 2$}
\par
\rotatebox{90}{\hspace{2cm}$\sigma = 3$}
\par
\rotatebox{90}{\hspace{2cm}$\sigma = 5$}
\par
\rotatebox{90}{\hspace{0.5cm}$\sigma = 10$}
\end{minipage}
\begin{minipage}{0.3\textwidth}
\centering 
%\vspace{0.53cm}
\includegraphics[height=3cm]{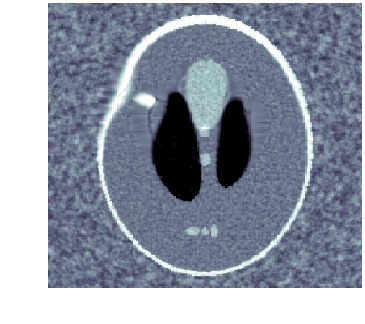}
%}

\par
 \includegraphics[height=3cm]{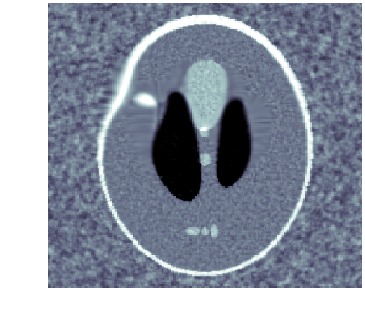}

\par
\includegraphics[height=3cm]{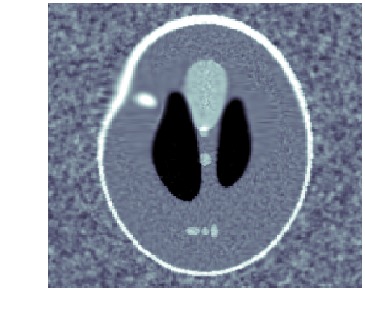}

\par
\includegraphics[height=3cm]{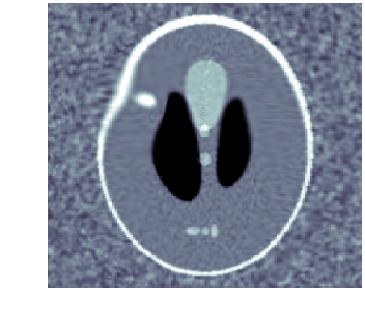}

\par
\includegraphics[height=3cm]{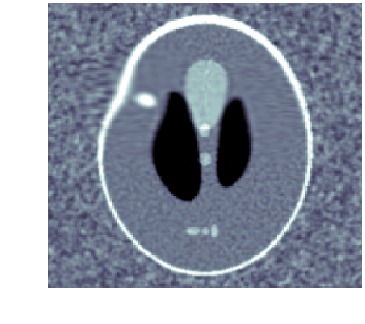}

\par
\includegraphics[height=3cm]{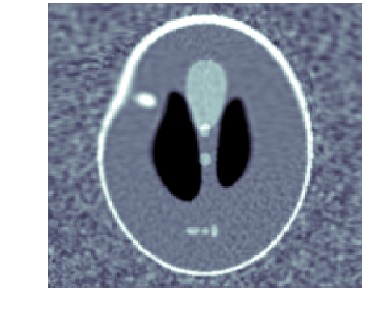}

\par
\includegraphics[height=3cm]{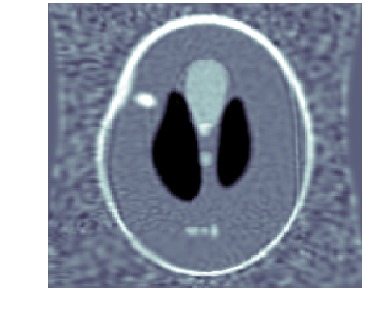}

\end{minipage}
\begin{minipage}{0.3\textwidth}
\centering 
%\vspace{0.53cm}
\includegraphics[height=3cm]{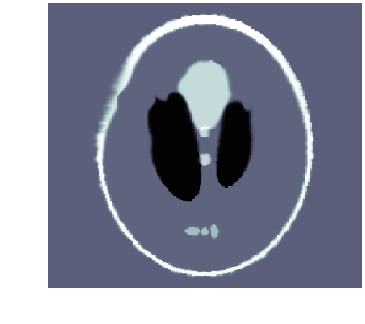}
%}

\par
 \includegraphics[height=3cm]{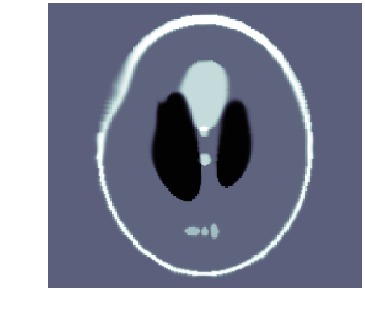}

\par
\includegraphics[height=3cm]{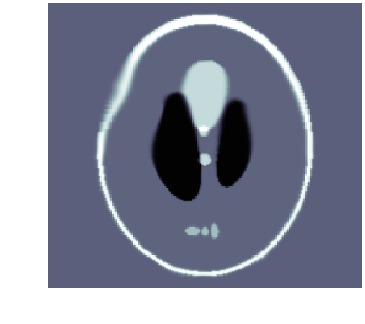}

\par
\includegraphics[height=3cm]{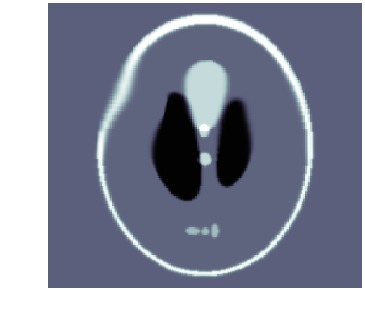}

\par
\includegraphics[height=3cm]{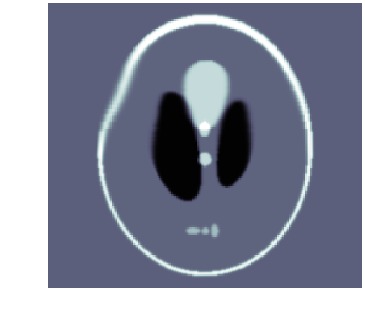}

\par
\includegraphics[height=3cm]{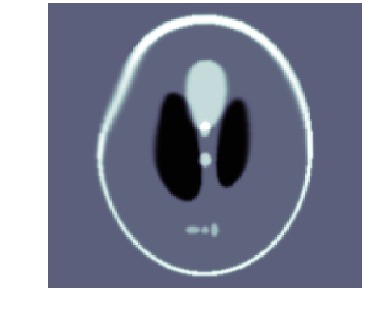}

\par
\includegraphics[height=3cm]{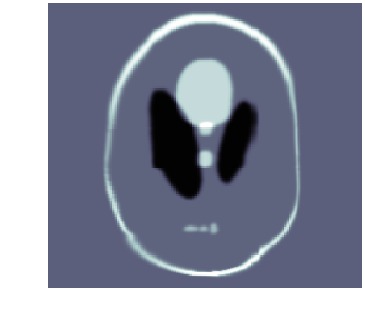}

\end{minipage}
\begin{minipage}{0.3\textwidth}
\centering 
%\vspace{0.53cm}
\includegraphics[height=3cm]{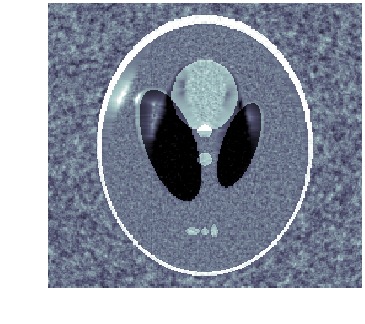}
%}

\par
 \includegraphics[height=3cm]{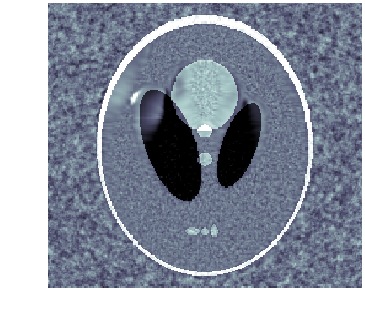}

\par
\includegraphics[height=3cm]{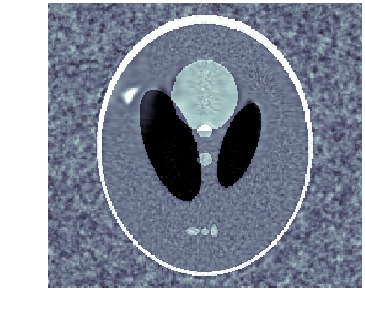}

\par
\includegraphics[height=3cm]{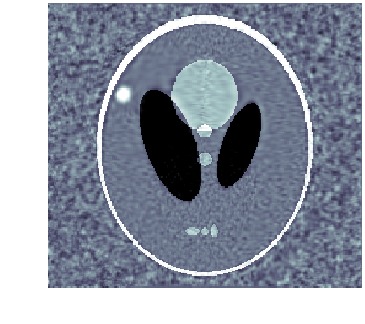}

\par
\includegraphics[height=3cm]{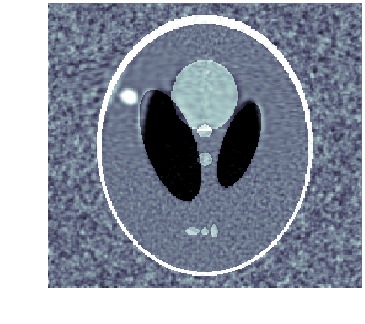}

\par
\includegraphics[height=3cm]{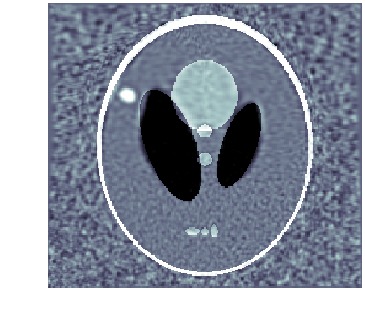}

\par
\includegraphics[height=3cm]{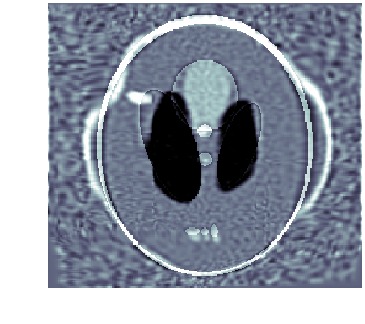}

\end{minipage}
\caption{Reconstruction results and their recomposition in template part and deformation part for various kernel size $\sigma$. \label{Fig:VariousSigma}}
\end{figure}

\subsection{Spatio-temporal reconstruction}\label{Sec:Test:SpatioTemporal}
The goal here is to recover the unknown temporal evolution of a template matched against (gated) parallel beam 2D ray transform data acquired at 10 different time points (from $t=0.1$ to $t=1$), so the target undergoes a temporal evolution.  
At each of the 10 time points, we only have limited tomographic data in the sense that $i$:th acquisition corresponds to sampling the parallel beam ray transform of the target at time $t_i$ using 10~angles randomly distributed in $[(i-1) \pi / 10, i \pi / 10]$ using $362$ lines/angles. 
Similarly to previous experiments, the reconstruction space is $\Omega = [-16, 16] \times [-16, 16]$, discretised as $256 \times 256$ pixel grey scale images.

The registration of the template $\template$ against the temporal series of data $\data_i$, $1 \leq i \leq 10$ at the $10$ time points $t_i$ is performed by minimizing the following functional with respect to \emph{one} trajectory $(\vfield,\intenfield) \in \VelocitySpace{2}{V \times \RecSpace}$:
\begin{equation*}
 \ObjFunc_{\gamma, \tau}(\vfield, \intenfield; \data_1, \dots, \data_{10}) 
   := \frac{\gamma}{2} \Vert \vfield \Vert_2^2 + \frac{\tau}{2} \Vert \intenfield \Vert_2^2 
        + \sum_{i=1}^{10} \DataDiscr\Bigl(\ForwardOp\bigl(\DeforOp(\diffeoflow{\vfield}{0,t_i},\templateflow{\vfield,\intenfield}{t_i})\bigr),\data_i \Bigr)
\end{equation*}
where $t \mapsto \templateflow{\vfield,\intenfield}{t}$, is the absolutely continuous solution to 
\begin{equation*}
\begin{cases}
  \dfrac{\der}{\der t} \templateflow{\vfield,\intenfield}{t}(x) = \intenfield\bigl(t, \diffeoflow{\vfield}{0,t}(x)\bigr) & \\[0.75em]
  \templateflow{\vfield,\intenfield}{0}(x) = \template(x)
\end{cases}
\quad\text{with $\diffeoflow{\vfield}{0,t} \in G_V$ as in \cref{eq.FlowDiffeo}.}
\end{equation*}

The target, the gated tomographic data, and the three trajectories (image, deformation and template) resulting from the metamorphosis based indirect registration are shown in \cref{Fig:TemporalExample}.
We see that metamorphosis based indirect registration can be used for spatio-temporal reconstruction even when (gated) data is highly under sampled. 
In particular, we can recover the evolution (both the geometric deformation and the appearance of the white disc) of the target. As a comparison, 
\cref{subfig:FBPTV} presents reconstructions obtained from \ac{FBP} and \ac{TV}. Here, data is a concatenation of the 10 gated data sets, thereby corresponding then sampling the ray transform using 100 angles in $[0, \pi]$. Note however that the temporal evolution of the target is not accounted for in these reconstructions.

\begin{figure}[h]
  \centering
  % Row 1
  \subfloat[The temporal evolution of the target.\label{subfig:TargetEvol}]{%
    \includegraphics[width=0.19\textwidth]{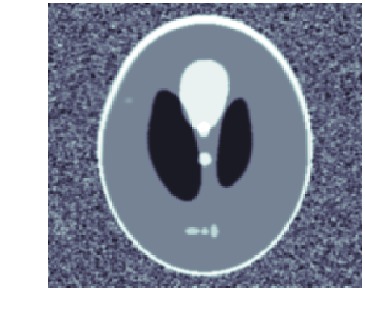}
    \includegraphics[width=0.19\textwidth]{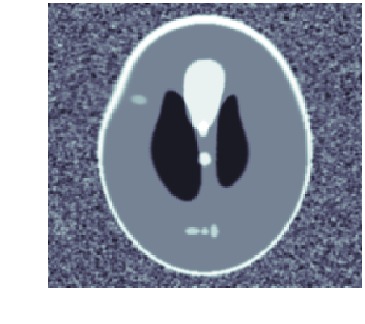}
    \includegraphics[width=0.19\textwidth]{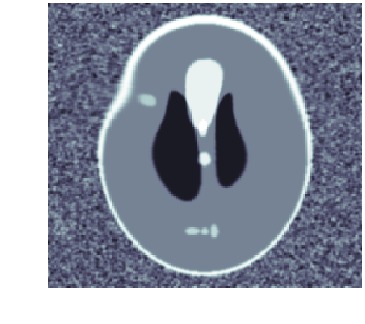}
    \includegraphics[width=0.19\textwidth]{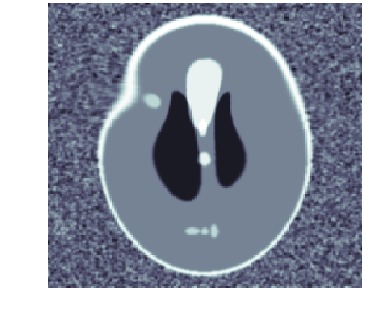}
    \includegraphics[width=0.19\textwidth]{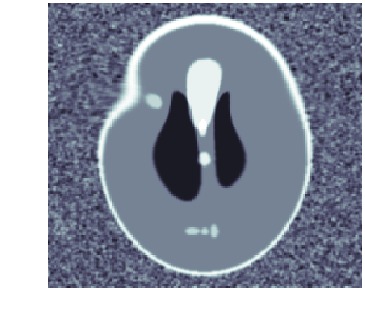}
    }
  \\        
  % Row 2
  \subfloat[The (gated) tomographic data. Each data set is highly incomplete (limited angle).\label{subfig:DataEvol}]{%  
    \includegraphics[width=0.19\textwidth]{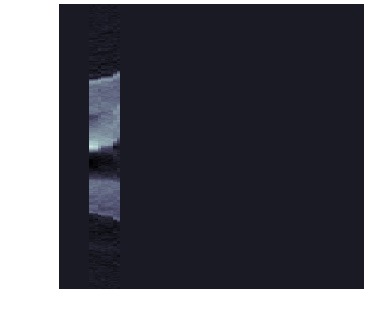}
    \includegraphics[width=0.19\textwidth]{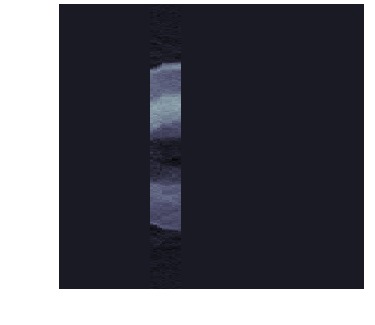}
    \includegraphics[width=0.19\textwidth]{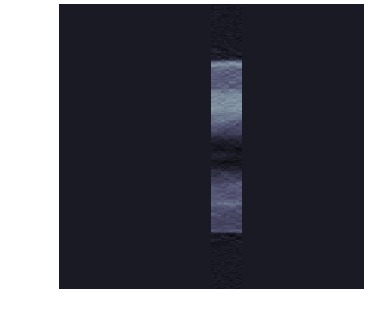}
    \includegraphics[width=0.19\textwidth]{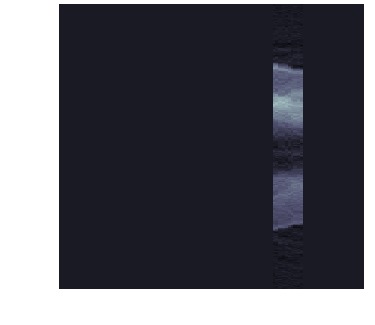}
    \includegraphics[width=0.19\textwidth]{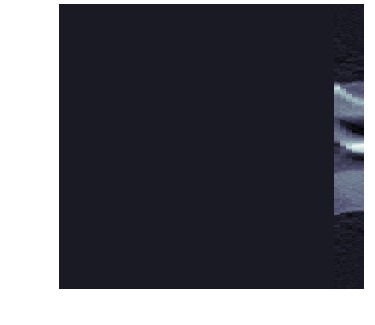}
  }  
  % Row 3
  \\    
  \subfloat[Image trajectory (reconstruction), combines deformation and template trajectories.\label{subfig:ImageEvol}]{%  
    \includegraphics[width=0.19\textwidth]{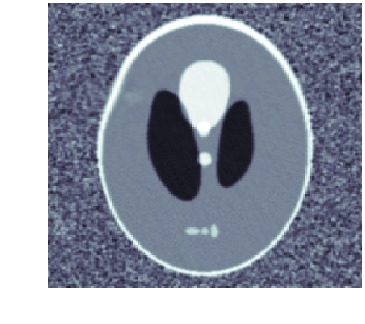}
    \includegraphics[width=0.19\textwidth]{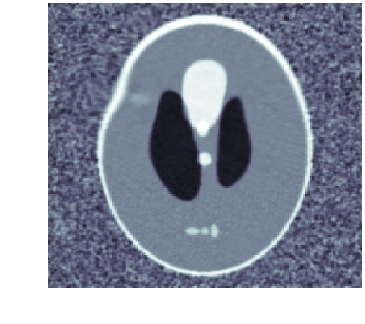}
    \includegraphics[width=0.19\textwidth]{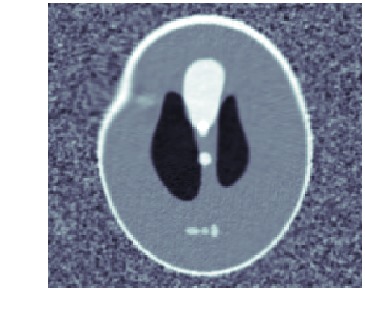}
    \includegraphics[width=0.19\textwidth]{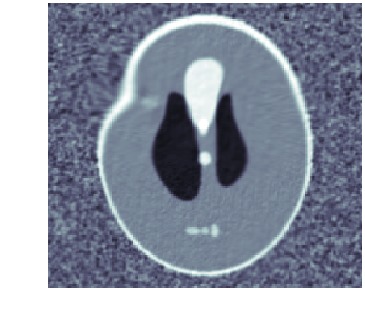}
    \includegraphics[width=0.19\textwidth]{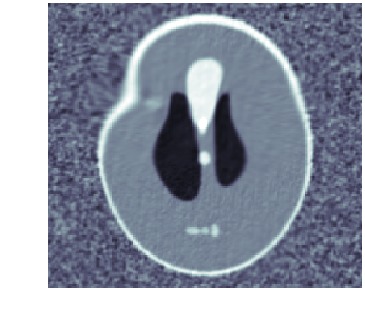}
  }
  \\
  % Row 4  
  \subfloat[Deformation trajectory, models mainly geometric changes.\label{subfig:DeforEvol}]{%  
    \includegraphics[width=0.19\textwidth]{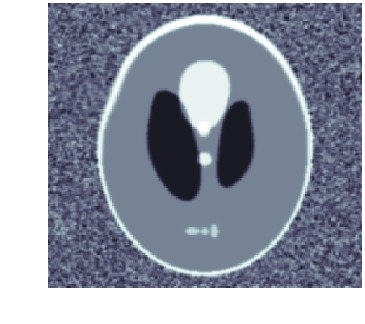}
    \includegraphics[width=0.19\textwidth]{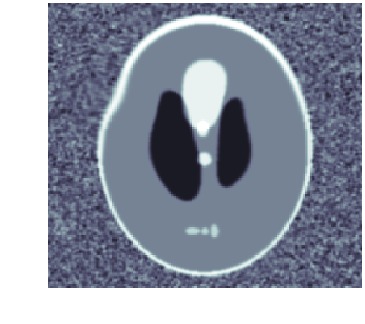}
    \includegraphics[width=0.19\textwidth]{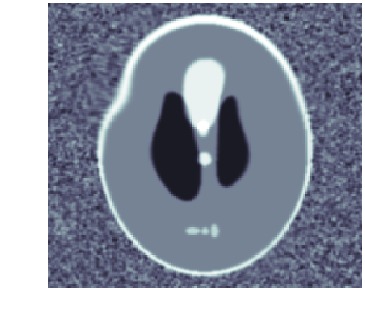}
    \includegraphics[width=0.19\textwidth]{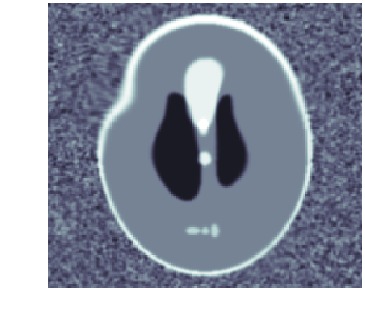}
    \includegraphics[width=0.19\textwidth]{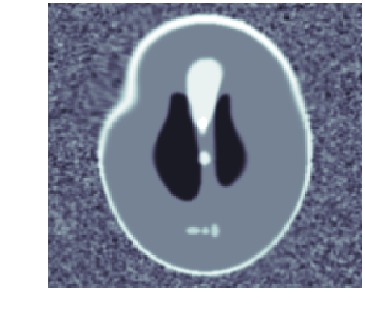}
  }    
  % Row 5  
  \\
  \subfloat[Template trajectory, models mainly intensity changes.\label{subfig:TemplateEvol}]{%  
    \includegraphics[width=0.19\textwidth]{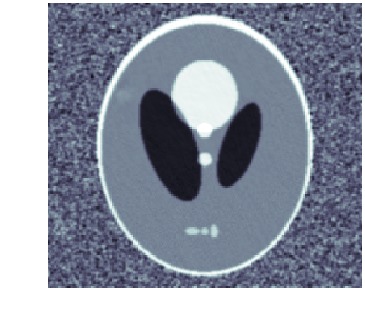}
    \includegraphics[width=0.19\textwidth]{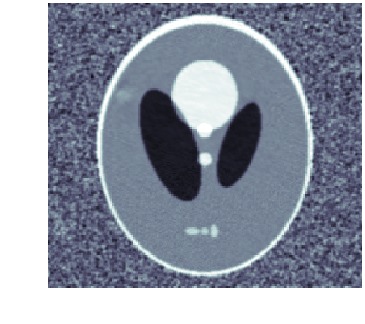}
    \includegraphics[width=0.19\textwidth]{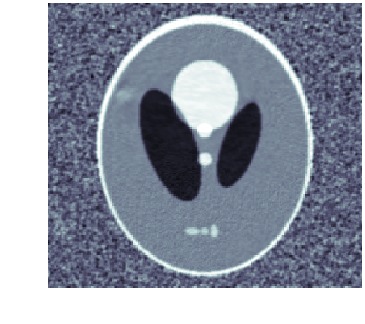}
    \includegraphics[width=0.19\textwidth]{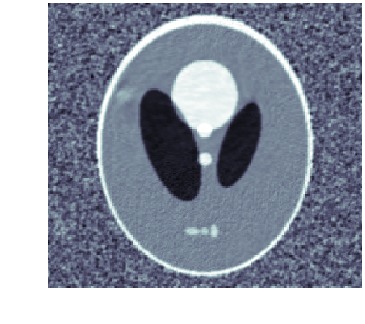}
    \includegraphics[width=0.19\textwidth]{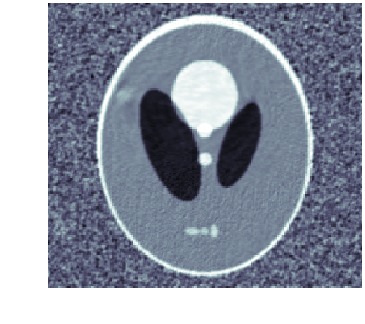}
  }      
  % Row 5  
  \\
  \subfloat[FBP (left) and TV (middle) reconstructions from concatenated data (right).\label{subfig:FBPTV}]{%  
    \includegraphics[width=0.25\textwidth]{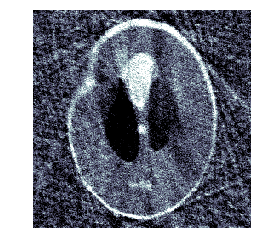}
    \includegraphics[width=0.25\textwidth]{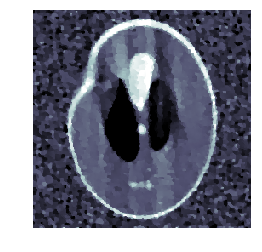}
    \includegraphics[width=0.27\textwidth]{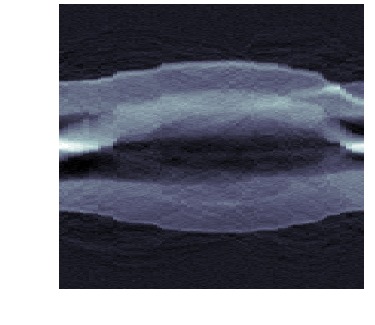}
  }    
\caption{Reconstructing the temporal evolution of a template using metamorphosis. Target \captionsubref{subfig:TargetEvol}, data \captionsubref{subfig:DataEvol}, and results \captionsubref{subfig:ImageEvol}--\captionsubref{subfig:TemplateEvol}, 
are shown at selected time points $t = 0.2, 0.4, 0.6, 0.8$, and $1.0$. As a comparison we show reconstructions assuming static target obtained from concatenating the gated tomographic data \captionsubref{subfig:FBPTV}.}
\label{Fig:TemporalExample}
\end{figure}

\section{Conclusions and discussion}

We introduced a metamorphosis-based framework for indirect registration and showed that this corresponds to a well-defined regularization method.
We also present several numerical examples from tomography. 

In particular, \cref{Sec:Test:SpatioTemporal} illustrates that this framework enables to recover the temporal evolution of a template from temporal data, even when data are very limited for each time point.
This approach assumes that one has access to an initial template.
In spatio-temporal reconstruction, such an initial template is unknown and it needs to be recovered as well. 
One approach for doing this is by an intertwined scheme that alternates between to steps (similarly to \cite{hinkle20124d}): 
\begin{inparaenum}[(i)]
\item given a template, estimate its evolution that is consistent with times series of data using the metamorphosis framework for indirect registration, and 
\item estimate the initial template from times series of data given its evolution.
\end{inparaenum}
The approach in \cref{Sec:Test:SpatioTemporal} solves the first step, which is the more difficult one. 

Another topic is the choice of hyperparameters. Our metamorphosis-based framework for indirect registration relies on three parameters, but as shown in \cref{Sec:Test:Robustness}, the most important one is the kernel-size $\sigma$. 
The latter has a strong influence on the way the reconstructed image trajectory decomposes into a  deformation and a template part. Clearly it acts as a regularisation parameter and a natural problem is to devise a scheme for choosing it depending on the size of features (scale) undergoing deformation. 
Unfortunately, similarly to direct registration using the \ac{LDDMM} framework, the choice of this parameter (and more generally choice of kernel for the \ac{RKHS} $V$) is still an open problem \cite{beg2005computing, Chen:2018aa, durrleman2014morphometry}. 
One way is to use a multi-scale approach \cite{bruveris2012mixture, risser2011simultaneous, sommer2011multi} but a general method for selecting an  appropriate kernel-size remains to be determined.

\section{Acknowledgements}
The work by Ozan \"Oktem, Barbara Gris and Chong Chen was supported by the Swedish Foundation for Strategic Research grant AM13-0049. 

\clearpage

\appendix

\section{Gradient computation}\label{Appendix:GradientComputation}
This section presents the computation of the gradient of $\ObjFunc_{\gamma, \tau}(\Cdot; \data)$, which is useful for any first order optimisation metod for minimising the functional $\ObjFunc_{\gamma, \tau}(\Cdot; \data)$ in \cref{Eq:ObjectiveFunctional}.
The computations assume  
\[  \template \in \RecSpace \cap C^1(\domain, \Real) 
    \quad\text{and}\quad
    (\vfield, \intenfield) \in \VelocitySpace{2}{V \times \RecSpace}
    \quad\text{with $\RecSpace = L^2(\domain,\Real)$.}
\]
Furthermore, for each $t \in [0, 1]$ we also assume $t \mapsto \intenfield(t,\Cdot) \in C^1(\domain, \Real)$.
In numerical implementations, we consider digitized images and considerations of the above type are not that restrictive. 

Let us first compute the differential of the data discrepancy term with respect to $\intenfield$ using the notation $\metatraj{\vfield,\intenfield}{t} := 
\DeforOp( \diffeoflow{\vfield}{t,0} , \templateflow{\vfield,\intenfield}{t}) =
\templateflow{\vfield,\intenfield}{t} \circ \diffeoflow{\vfield}{t,0}$. 
As noted in \cref{Eq:TemplateEvolutionSolution}, we have  
\begin{equation}\label{Eq:TemplateEvolutionSolution2}
 \metatraj{\vfield,\intenfield}{t}(x)= (\templateflow{\vfield,\intenfield}{t} \circ \diffeoflow{\vfield}{t,0})(x) 
      = \template\bigl( \diffeoflow{\vfield}{t,0}(x) \bigr) + \int_0^t \intenfield\bigl( \tau, \diffeoflow{\vfield}{t,\tau}(x) \bigr) \der \tau. 
\end{equation}
Then 
\begin{multline*}
\partial_{\intenfield}\Bigl[ \DataDiscr\bigl( \metatraj{\vfield,\intenfield}{t}, \data \bigr) \Bigr](\intenfield) (\eta) 
  = \Bigl\langle 
        \nabla\DataDiscr(\metatraj{\vfield,\intenfield}{t}, \data), 
        \partial_{\intenfield}\metatraj{\vfield,\intenfield}{t}(\intenfield)(\eta) 
      \Bigr\rangle
\\
 = \int_\domain \int_0^t  
     \nabla \DataDiscr\bigl(\metatraj{\vfield,\intenfield}{1}, \data \bigr) \eta(\tau, \diffeoflow{\vfield}{t,\tau}(x)) \der \tau \der x 
\\
 = \int_\domain  \int_0^1 1_{\tau \leq t} \vert \Det(\der \diffeoflow{\vfield}{\tau,t} (x)) \vert  
      \nabla  \DataDiscr(\metatraj{\vfield,\intenfield}{t},\data) (\diffeoflow{\vfield}{\tau,t} (x)) \eta(\tau, x) \der \tau  \der x  
\\      
 = \Bigl\langle 1_{\cdot \leq t} \vert \Det(\der \diffeoflow{\vfield}{\cdot, t} ) \vert  \nabla  \DataDiscr(\metatraj{\vfield,\intenfield}{t},\data))(\diffeoflow{\vfield}{\cdot, t} ), \eta \Bigr\rangle_{\VelocitySpace{2}{L^2(\domain, \Real)}}.
\end{multline*}

In order to compute the differential of the discrepancy term with respect to $\vfield$, we start by computing the differential of 
$\metatraj{\vfield,\intenfield}{1}$ with respect to $\vfield$.
Hence, let $\vfieldother \in \VelocitySpace{2}{V}$ and $x \in \domain$. Then 
\begin{multline}\label{eq:ImageTrajDerivative}
\frac{\der}{\der \epsilon} \metatraj{\vfield + \epsilon \vfieldother,\intenfield}{t}(x) \Bigl\vert_{\epsilon = 0} 
  =  \Bigl\langle 
         \nabla \template\bigl(\diffeoflow{\vfield}{t,0}(x) \bigr), 
         \frac{\der}{\der \epsilon} \diffeoflow{\vfield + \epsilon \vfieldother}{t,0}(x) \bigl\vert_{\epsilon=0} 
       \Bigr\rangle 
\\[0.5em]
\shoveleft{\qquad\qquad 
   + \int_0^t \Bigl\langle 
         \nabla \intenfield(\tau, \diffeoflow{\vfield}{t,\tau}(x)),%
          \frac{\der}{\der \epsilon}{\diffeoflow{\vfield + \epsilon \vfieldother}{t,\tau}(x)}\bigl\vert_{\epsilon=0}  
       \Bigr\rangle \der \tau 
}\\        
\shoveleft{\qquad  
   = - \int_0^t  \Bigl\langle 
          \nabla \template (\diffeoflow{\vfield}{t,0} (x)),\der \diffeoflow{\vfield}{s,0}{\bigl(\diffeoflow{\vfield}{t,s}(x)\bigr)}%
          \Bigl(\vfieldother(s, \diffeoflow{\vfield}{t,s}(x)\bigr)\Bigr) 
        \Bigr\rangle \der s 
}\\        
\shoveleft{\qquad\qquad
    \int_0^t \biggl\langle 
         \nabla \intenfield\bigl(\tau, \diffeoflow{\vfield}{t,\tau}(x) \bigr),  
          \int_ t^\tau \der{\diffeoflow{\vfield}{ s,\tau} }{\bigl(\diffeoflow{\vfield}{t,s}(x)\bigr)}\Bigl(\vfieldother\bigl(s, \diffeoflow{\vfield}{t,s}(x)\bigr)\Bigr) 
          \der s \biggr\rangle  \der \tau 
}\\        
\shoveleft{\qquad  
   = - \int_0^t  \biggl\langle 
          \nabla \template\bigl(\diffeoflow{\vfield}{t,0} (x)\bigr),
          \der{\diffeoflow{\vfield}{s,0}}{\bigl(\diffeoflow{\vfield}{t,s}(x)\bigr)} 
          \Bigl((\vfieldother\bigl((s, \diffeoflow{\vfield}{t,s}(x)\bigr)\Bigr) 
        \biggr\rangle \der s 
}\\        
\shoveleft{\qquad\qquad
 - \int_0^t \int_0^s \biggl\langle  
        \nabla \intenfield(\tau, \Cdot) \circ \diffeoflow{\vfield}{t,\tau}(x), 
         \der{\diffeoflow{\vfield}{s,\tau}}{\bigl(\diffeoflow{\vfield}{t,s}(x)\bigr)} 
         \Bigl(\vfieldother\bigl(s,\diffeoflow{\vfield}{t,s}(x)\bigr)\Bigr) 
       \biggr\rangle \der \tau \der s.
}         
\end{multline} 
Using \cref{eq:ImageTrajDerivative}, we can compute the derivative of $\epsilon \mapsto \DataDiscr\bigl(\DeforOp(\diffeoflow{\vfield+\epsilon \vfieldother}{t}, \templateflow{\vfield+\epsilon \vfieldother,\intenfield}{t})\bigr)$ at $\epsilon=0$:
\begin{multline*}
\frac{\der}{\der \epsilon} \DataDiscr\bigl( \DeforOp\bigl(\diffeoflow{\vfield+ \epsilon \vfieldother}{t}, \templateflow{\vfield+ \epsilon \vfieldother,\intenfield}{t}\bigr) \bigr)\bigl\vert_{\epsilon=0} 
  = \Bigl\langle 
       \nabla\DataDiscr\bigl( \metatraj{\vfield,\intenfield}{t}, \data \bigr), 
  \frac{\der}{\der \epsilon} \metatraj{\vfield+\epsilon \vfieldother,\intenfield}{t}\bigl\vert_{\epsilon=0}\Bigr\rangle  
\\
\shoveleft{\qquad  
= - \int_\domain \Bigg\{ 
     \int_0^t  \nabla\DataDiscr\bigl( \metatraj{\vfield,\intenfield}{t}, \data \bigr)(x) \cdot 
       \Bigg[ \Bigl\langle 
        \nabla\template \bigl(\diffeoflow{\vfield}{t,0}(x)\bigr),
        \der{\diffeoflow{\vfield}{s,0}}{\bigl(\diffeoflow{\vfield}{t,s}(x)\bigr)} (\vfieldother(s, \diffeoflow{\vfield}{t,s}(x))) 
     \Bigr\rangle
}\\        
\shoveleft{\qquad\qquad
     +  \int_0^s \biggl\langle  
           \nabla\zeta (\tau, \cdot) \circ \diffeoflow{\vfield}{t,\tau}(x), 
           \der{\diffeoflow{\vfield}{s,\tau}}{\bigl(\diffeoflow{\vfield}{t,s}(x)\bigr)}
           \Bigl(\vfieldother\bigl(s,\diffeoflow{\vfield}{t,s}(x)\bigr)\Bigr) 
         \biggr\rangle \der \tau 
     \Biggr] \der s 
\Biggr\} \der x 
}\\
\shoveleft{\qquad  
= - \int_\domain \Bigg\{ 
       \int_0^t \Bigl\vert \Det\bigl(\der{\diffeoflow{\vfield}{s,t}}(x)\bigr) \Bigr\vert 
            \nabla\DataDiscr\bigl( \metatraj{\vfield,\intenfield}{t}, \data \bigr)\bigl(\diffeoflow{\vfield}{s,t}(x)\bigr) \cdot 
}\\        
\shoveleft{\qquad\qquad\qquad\qquad
       \biggl[ \Bigl\langle \nabla\template\bigl(\diffeoflow{\vfield}{s,0}(x)\bigr),
          \der{\diffeoflow{\vfield}{s,0}}(x)\bigl(\vfieldother(s,x)\bigr) 
         \Bigr\rangle  
}\\        
\shoveleft{\qquad\qquad\qquad\qquad\qquad
         + \int_0^s \Bigl\langle  
              \nabla\zeta (\tau, \cdot) \circ \diffeoflow{\vfield}{s,\tau}(x), 
              \der{\diffeoflow{\vfield}{s,\tau}}(x) \bigl(\vfieldother(s,x)\bigr) 
            \Bigr\rangle  \der \tau 
  \biggr] \der s \Biggr\} \der x 
}\\
\shoveleft{\qquad  
   =  -  \int_\domain \int_0^t \Bigl\vert \Det\bigl(\der{\diffeoflow{\vfield}{s,t}}(x)\bigr) \Bigr\vert 
              \nabla\DataDiscr\bigl( \metatraj{\vfield,\intenfield}{t}, \data \bigr)\bigl(\diffeoflow{\vfield}{s,t}(x)\bigr) 
%}\\
%\shoveleft{\qquad\qquad
  \cdot \bigg[
    \bigl\langle 
      \nabla(\template \circ \diffeoflow{\vfield}{s,0} )(x), 
      \vfieldother(s,x) 
    \bigr\rangle
}\\        
\shoveleft{\qquad\qquad
    + \int_0^s 
         \Bigl\langle \nabla(\zeta( \tau, \cdot ) \circ \diffeoflow{\vfield}{s,\tau})(x), \vfieldother(s,x) \Bigr\rangle  \der \tau \Biggr] \der s  \der x 
}\\
\shoveleft{\qquad  
   = - \int_\domain \int_0^1 
   \biggl\langle 
         1_{s \leq t} \Bigl\vert \Det\bigl(\der{\diffeoflow{\vfield}{s,t}}(x)\bigr) \Bigr\vert \nabla\DataDiscr\bigl( \metatraj{\vfield,\intenfield}{t}, \data \bigr)\bigl(\diffeoflow{\vfield}{s,t}(x)\bigr)  
}\\        
\shoveleft{\qquad\qquad
         \biggl[ \nabla(\template \circ \diffeoflow{\vfield}{s,0} )(x) + 
         \int_0^s   \nabla(\zeta( \tau, \cdot ) \circ \diffeoflow{\vfield}{s,\tau})(x)  \der \tau \biggr], 
         \vfieldother(s,x) 
   \biggr\rangle \der s \der x 
}\\
\shoveleft{\qquad  
= - \biggl\langle  
    1_{ \cdot \leq t} \Bigl\vert \Det(\der{\diffeoflow{\vfield}{\cdot,t}}) \Bigr\vert 
    \nabla\DataDiscr\bigl( \metatraj{\vfield,\intenfield}{t}, \data \bigr) \circ \diffeoflow{\vfield}{\cdot,t}   
}\\        
\shoveleft{\qquad\qquad\qquad
     \biggl[ \nabla(\template \circ \diffeoflow{\vfield}{\cdot,0})(\cdot)  + 
     \int_0^{\cdot} \nabla(\zeta( \tau, \cdot ) \circ \diffeoflow{\vfield}{\cdot, \tau})(\cdot) \der \tau \biggr], 
   \vfieldother 
\biggr\rangle_{L^2([0,1],L^2(\domain,\Real^d))} 
}\\
\shoveleft{\qquad  
= - \biggl\langle \int_\domain K(x, \cdot ) 
   1_{ \cdot  \leq t} \Bigl\vert \Det(\der{ \diffeoflow{\vfield}{\cdot,t}}(x)) \Bigr\vert   
   \nabla\DataDiscr\bigl( \metatraj{\vfield,\intenfield}{t}, \data \bigr)\bigl(\diffeoflow{\vfield}{\cdot,t}(x)\bigr) 
}\\        
\shoveleft{\qquad\qquad\qquad
 \biggl[ \nabla(\template \circ \diffeoflow{\vfield}{\cdot,0}) (x)  + 
 \int_0^{\cdot}   \nabla(\zeta( \tau, \cdot ) \circ \diffeoflow{\vfield}{\cdot, \tau} ) (x) \der \tau  \biggr] ,\vfieldother 
\biggr\rangle_{L^2([0,1],V)}.
} 
\end{multline*}

\bibliographystyle{abbrv}
\bibliography{biblio}

\end{document}